\documentclass{article}

\usepackage[final]{neurips_2020}

\usepackage[utf8]{inputenc} % allow utf-8 input
\usepackage[T1]{fontenc}    % use 8-bit T1 fonts
\usepackage{hyperref}       % hyperlinks
\usepackage{url}            % simple URL typesetting
\usepackage{booktabs}       % professional-quality tables
\usepackage{amsfonts}       % blackboard math symbols
\usepackage{nicefrac}       % compact symbols for 1/2, etc.
\usepackage{microtype}      % microtypography
\usepackage{graphicx}
\usepackage{subfigure}
\usepackage{multirow}
\usepackage{xcolor}
\definecolor{darkgreen}{rgb}{0,0.5,0}
\usepackage{hyperref}
\usepackage{natbib}
\usepackage{amssymb,amsmath,amsthm}
\newtheorem{theorem}{Theorem}
\graphicspath{ {./figs/} }
\usepackage{booktabs} % for professional tables
\usepackage[capitalize]{cleveref}

\usepackage{wrapfig}
\usepackage{selectp}
\usepackage{algorithm,algorithmicx,algpseudocode}
\newcommand{\Tau}{\mathrm{T}}

\title{A Self-Tuning Actor-Critic Algorithm}

\author{Tom Zahavy, Zhongwen Xu, Vivek Veeriah, Matteo Hessel, Junhyuk Oh,\\ \textbf{Hado van Hasselt, David Silver and Satinder Singh}\\
\textbf{Deepmind}\\
\{tomzahavy,zhongwen,vveeriah,mtthss,junhyuk,hado,davidsilver,baveja\}@google.com}
\begin{document}

\maketitle

\begin{abstract}
    Reinforcement learning algorithms are highly sensitive to the choice of hyperparameters, typically requiring significant manual effort to identify hyperparameters that perform well on a new domain. In this paper, we take a step towards addressing this issue by using metagradients to automatically adapt hyperparameters online by meta-gradient descent (Xu et al., 2018). We apply our algorithm, Self-Tuning Actor-Critic (STAC), to self-tune all the differentiable hyperparameters of an actor-critic loss function, to discover auxiliary tasks, and to improve off-policy learning using a novel leaky V-trace operator. STAC is simple to use, sample efficient and does not require a significant increase in compute. Ablative studies show that the overall performance of STAC improved as we adapt more hyperparameters. When applied to the Arcade Learning Environment (Bellemare et al. 2012), STAC improved the median human normalized score in $200$M steps from $243\%$ to $364\%$. When applied to the DM Control suite (Tassa et al., 2018), STAC improved the mean score in $30$M steps from $217$ to $389$ when learning with features, from $108$ to $202$ when learning from pixels, and from $195$ to $295$ in the Real-World Reinforcement Learning Challenge (Dulac-Arnold et al., 2020).
\end{abstract}

\section{Introduction}
Deep Reinforcement Learning (RL) algorithms often have many modules and loss functions with many hyperparameters. When applied to a new domain, these hyperparameters are searched via cross-validation, random search \citep{RSHPO}, or population-based training \citep{jaderberg2017population}, which requires extensive computing resources. Meta-learning approaches in RL (e.g., MAML, \citet{finn2017model}) focus on learning good initialization via multi-task learning and transfer. However, many of the hyperparameters must be adapted during the agent's lifetime to achieve good performance (learning rate scheduling, exploration annealing, etc.). This motivates a significant body of work on specific solutions to tune specific hyperparameters, within a single agent \textit{lifetime} \citep{schaul2019adapting, mann2016adaptive, white2016greedy,rowland2019adaptive,sutton1992adapting}. 

% NIPS2019_8710
Metagradients, on the other hand, provide a general and compute-efficient approach for self-tuning in a single lifetime. The general concept is to represent the training loss as a function of both the agent parameters and the hyperparameters. The agent optimizes the parameters to minimize this loss function, w.r.t the current hyperparameters. The hyperparameters are then self-tuned via backpropagation to minimize a fixed loss function. This approach has been used to learn the discount factor or the $\lambda$ coefficient  \citep{xu2018meta}, to discover intrinsic rewards \citep{zheng2018learning} and auxiliary tasks \citep{veeriah2019discovery}. Finally, we note that there also exist derivative-free approaches for self-tuning hyper parameters \citep{NIPS2019_8710, tang2020online}. 

% Main contributions

This paper makes the following contributions. \textbf{First,} we introduce two novel ideas that extend IMPALA~\citep{espeholt2018impala} with additional components. \textbf{(1)} The first agent, referred to as a Self-Tuning Actor-Critic (STAC), self-tunes all the differentiable hyperparameters in the IMPALA loss function. In addition, STAC introduces a \textit{leaky V-trace} operator that mixes importance sampling (IS) weights with truncated IS weights. The mixing coefficient in leaky V-trace is differentiable (unlike the original V-trace) but similarly balances the variance-contraction trade-off in off-policy learning. \textbf{(2)} The second agent, \textit{STACX} (STAC with auXiliary tasks), adds auxiliary parametric actor-critic loss functions to the loss function and self-tunes their metaparameters. STACX self-tunes the discount factors of these auxiliary losses to different values than those of the main task, helping it to reason about multiple horizons. 

\textbf{Second,} we demonstrate empirically that self-tuning consistently improves performance. When applied to the Arcade Learning Environment \citep[ALE]{bellemare2013arcade}, STAC improved the median human normalized score in $200$M steps from $243\%$ to $364\%$. When applied to the DM Control suite \citep{tassa2018deepmind}, STAC improved the mean score in $30$M steps from $217$ to $389$ when learning with features, from $108$ to $202$ when learning from pixels, and from $195$ to $295$ in the Real-World Reinforcement Learning Challenge \citep{dulac2020empirical}. 

We conduct extensive ablation studies, showing that the performance of STACX consistently improves as it self-tunes more hyperparameters; and that STACX improves the baseline when self-tuning different subsets of the metaparameters. STACX performs considerably better than previous metagradient algorithms \citep{xu2018meta,veeriah2019discovery} and across a broader range of environments.

\textbf{Finally,} we investigate the properties of STACX via a set of experiments. (1) We show that STACX is more robust to its hyperparameters than the IMPALA baseline. (2) We visualize the self-tuned metaparameters through training and identify trends. (3) We demonstrate a tenfold scale up in the number of self-tuned hyperparameters -- 21 compared to two in \citep{xu2018meta}. This is the most significant number of hyperparameters tuned by meta-learning at scale and does not require a significant increase in compute (see \cref{table:runtime} in the supplementary and the discussion that follows it).

\section{Background}
\label{sec:metagrad}

We begin with a brief introduction to actor-critic algorithms and IMPALA \citep{espeholt2018impala}. Actor-critic agents maintain a policy $\pi_\theta(a|x)$ and a value function $V_\theta(x)$ that are parameterized with parameters $\theta.$ These policy and the value function are trained via an actor-critic update rule, with a policy gradient loss and a value prediction loss. In IMPALA, we additionally add an entropy regularization loss. The update is represented as the gradient of the following pseudo-loss function
\begin{align}
\label{eq:vtrace_loss}
    L&_{\text{Value}}(\theta)  =  \sum\nolimits_{s \in \Tau}\left(v_s - V_\theta(x_s)\right)^2  \nonumber \\
    L&_{\text{Policy}}(\theta) = -  \sum\nolimits_{s \in \Tau} \rho_s
     \log \pi_\theta (a_s|x_s) (r_s+\gamma v_{s+1} - V_\theta(x_s)) \nonumber \\
    L&_{\text{Entropy}}(\theta) = - \sum\nolimits_{s \in \Tau}
     \sum\nolimits_a \pi_\theta (a|x_s) \log \pi_\theta (a|x_s) \nonumber \\
     L& (\theta)  =  g_v L_{\text{Value}}(\theta)
    + g_p L_{\text{Policy}}(\theta) + 
    g_e L_{\text{Entropy}}(\theta).
\end{align}

In each iteration $t,$ the gradients of these losses are computed on data $\Tau$ that is composed from a mini batch of $m$ trajectories, each of size $n$ (see the the supplementary material for more details). We refer to the policy that generates this data as the behaviour policy $\mu(a_s|x_s),$ where the superscript $s$ will refer to the time index within a trajectory. In the on policy case, $\mu(a_s|x_s) = \pi(a_s|x_s)$, $\rho_s = 1$ , and we have that $v_s$ is the n-steps bootstrapped return $v_s = \sum_{j=s}^{s+n-1} \gamma ^{j-s}r_j + \gamma ^n V(x_{s+n}).$  

IMPALA uses a distributed actor critic architecture, that assigns copies of the policy parameters to multiple actors in different machines to achieve higher sample throughput. As a result, the target policy $\pi$ on the learner machine can be several updates ahead of the actor’s policy $\mu$ that generated the data used in an update. Such off policy discrepancy can lead to biased updates, requiring us to multiply the updates with importance sampling (IS) weights for stable learning. Specifically, IMPALA \citep{espeholt2018impala} uses truncated IS weights to balance the variance-contraction trade-off on these off-policy updates. This corresponds to instantiating \cref{eq:vtrace_loss} with 
\begin{align}
    \label{eq:vtrace}
    v_s &= V(x_s)+ \sum\nolimits_{j=s}^{s+n-1} \gamma ^{j-s} \left(\Pi ^{j-1}_{i=s}c_i \right)\delta_j V, 
    \enspace \delta _j V = \rho_j (r_j + \gamma V(x_{j+1}) - V(x_j)) \nonumber \\
    \rho _j  &= \min \left(\bar \rho, \frac{\pi(a_j|x_j)}{\mu(a_j|x_j)} \right), c _i = \lambda \min \left(\bar c, \frac{\pi(a_i|x_i)}{\mu(a_i|x_i)}\right).
\end{align}

\textbf{Metagradients.}
In the following, we  consider three types of parameters: $\theta$ -- the agent parameters; $\zeta$ -- the hyperparameters; $\eta \subset \zeta$ -- the metaparameters. 
$\theta$ denotes the parameters of the \textbf{agent} and parameterizes, for example, the value function and the policy; these parameters are randomly initialized at the beginning of an agent's lifetime and updated using backpropagation on a suitable \textit{inner} loss function. $\zeta$ denotes the \textbf{hyperparameters}, including, for example, the parameters of the optimizer (e.g., the learning rate) or the parameters of the loss function (e.g., the discount factor); these may be tuned throughout many lifetimes (for instance, via random search) to optimize an \textit{outer} (validation) loss function. Typical deep RL algorithms consider only these first two types of parameters. In metagradient algorithms a third set of parameters is specified: the \textbf{metaparameters}, denoted $\eta$, which are a \emph{subset} of the differentiable parameters in $\zeta.$ Starting from some initial value (itself a hyperparameter), they are then self-tuned during training within a single lifetime.

Metagradients are a general framework for adapting, online, within a single lifetime, the differentiable hyperparameters $\eta$. Consider an inner loss that is a function of both the parameters $\theta$ and the metaparameters $\eta$: $L_{\text{inner}}(\theta;\eta)$. On each step of an inner loop, $\theta$ can be optimized with a fixed $\eta$ to minimize the inner loss $L_{\text{inner}}(\theta;\eta),$ by updating $\theta$ with the following gradient $\tilde \theta(\eta_t) \doteq \theta_{t+1} = \theta_t - \nabla_\theta L_{\text{inner}}(\theta_t;\eta_t).$

In an outer loop, $\eta$ can then be optimized to minimize the outer loss by taking a metagradient step. As $\tilde \theta (\eta)$ is a function of $\eta$ this corresponds to updating the $\eta$ parameters by differentiating the outer loss w.r.t $\eta,$ such that $\eta_{t+1} = \eta_t - \nabla_\eta L_{\text{outer}}(\tilde \theta(\eta_t)).$
The algorithm is general and can be applied, in principle, to any differentiable meta-parameter $\eta$ used by the inner loss. Explicit instantiations of the metagradient RL framework require the specification of the inner and outer loss functions. 

% It is perhaps surprising that we may choose to optimize a different loss function in the inner loop, instead of the outer loss we ultimately care about. However, this is not a new idea. Regularization, for example, is a technique that changes the training loss to balance the bias-variance trade-off and avoid overfitting. In model-based RL, it was shown that the policy found using a smaller discount factor can be better than a policy learned with the true discount factor \citep{jiang2015dependence}. Auxiliary tasks \citep{jaderberg2016reinforcement} are another example, where gradients are taken w.r.t unsupervised loss functions to improve the agent's representation. Finally, it is well known that to maximize the long term cumulative reward efficiently (the objective of the outer loop), RL agents must explore, i.e., act according to a different objective (accounting, for instance, for uncertainty). 

\section{Self-Tuning actor-critic agents}
We now describe the inner and outer loss of our agent. The general idea is to self-tune all the differentiable hyperparameters. The outer loss is the original IMPALA loss (\cref{eq:vtrace_loss}) with an additional Kullback–Leibler (KL) term, which regularizes the $\eta$-update not to change the policy:  

\begin{align}
\label{eq:implala_outer_loss}
     L_{\text{outer}}(\tilde \theta (\eta))
     = g_v^{\text{outer}} L_{\text{Value}}(\theta)
     + g_p^{\text{outer}}  L_{\text{Policy}}(\theta) 
     + g_e^{\text{outer}} L_{\text{Entropy}}(\theta)
     + g^{\text{outer}}_{\text{kl}} \text{KL} (\pi_{\tilde \theta(\eta)}, \pi_{\theta}).
\end{align}

The inner loss function, is parametrized by the metaparameters $\eta = \{ \gamma, \lambda, g_v, g_p, g_e \}$:

\begin{align}
\label{eq:inner_loss}
    L(\theta; \eta)  =  g_v L_{\text{Value}}(\theta)
    + g_p L_{\text{Policy}}(\theta) + 
    g_e L_{\text{Entropy}}(\theta),
\end{align}

Notice that $\gamma$ and $\lambda$ affect the inner loss through the definition of $v_s$ in \cref{eq:vtrace}. The loss coefficients $g_v, g_p, g_e$ allow for loss specific learning rates and support dynamically balancing exploration with exploitation by adapting the entropy loss weight. We apply a sigmoid activation on all the metaparameters, which ensures that they remain bounded. We also multiply the loss coefficients ($g_v, g_e, g_p$) by the respective coefficient in the outer loss to guarantee that they are initialized from the same values. For example, $\gamma = \sigma(\gamma), g_v = \sigma(g_v) g_v ^{\text{outer}}$. The exact details can be found in the supplementary (\cref{alg:stac}, line 11).

\begin{wraptable}{r}{0.5\textwidth}
\begin{center}
    \resizebox{0.455\textwidth}{!}{
    \begin{tabular}[t]{|l|l|l|}
    \hline
        & IMPALA & STAC \\ \hline
        $\theta$  & $V_\theta, \pi_\theta $ &  $V_\theta, \pi_\theta $ \\ \hline
        \multirow{3}{*}{$\zeta$}
        & \multirow{3}{*}{$\{ \gamma, \lambda, g_v, g_p, g_e \}$}
        & $\{ \gamma^{\text{outer}}, \lambda^{\text{outer}}, g^{\text{outer}}_v, g^{\text{outer}}_p, g^{\text{outer}}_e \}$ \\ 
        && Initialisations\\ 
        && Meta optimizer parameters, $g^{\text{outer}}_{\text{kl}}$ \\ \hline
        $\eta$    & $\text{--}$ & $ \{ \gamma, \lambda, g_v, g_p, g_e \}$ \\ \hline
    \end{tabular}}
    \end{center}
\caption{Parameters in IMPALA and STAC.}
\label{table:params}
\end{wraptable}

\cref{table:params} summarizes all the \textbf{hyperparameters} that are required for STAC. STAC has new hyperparameters (compared to IMPALA), but we found that using simple ``rules of thumb'' is sufficient to tune them. These include the initializations of the metaparameters, the hyperparameters of the outer loss, and meta optimizer parameters. The exact values can be found in the supplementary  (see \cref{table:hyperparameters}. For the outer loss hyperparameters, we use exactly the same hyperparameters that were used in the IMPALA paper for all of our agents ($g_v^{\text{outer}}=0.25, g_p^{\text{outer}}=1, g_v^{\text{outer}}=1, \lambda^{\text{outer}}=1$), with one exception: we use $\gamma = 0.995$ as it was found in \citep{xu2018meta} to improve in Atari the performance of IMPALA and the metagradient agent in Atari, and $\gamma = 0.99$ in DM control suite.

For the initializations of the metaparameters we use the corresponding parameters in the outer loss, i.e., for any metaparameter $\eta_i,$ we set $\eta_i^{\text{Init}} = 4.6$ such that $\sigma(\eta_i^{\text{Init}}) = 0.99.$
This guarantees that the inner loss is initialized to be (almost) the same as the outer loss. The exact value was chosen arbitrarily, and we later show in \cref{fig:rob_init} that the algorithm is not sensitive to it. For the meta optimizer, we use ADAM with default settings (e.g., learning rate is set to $10^{-3}$), and for the the KL coefficient, we use $g^{\text{outer}}_{\text{kl}}=1$). 

\subsection{STAC and leaky V-trace} 
The hyperparameters that we considered for self-tuning so far, $\eta = \{ \gamma, \lambda, g_v, g_p, g_e \},$ parametrized the loss function in a differentiable manner. The truncation levels in the V-trace operator, on the other hand, are nondifferentiable. We now introduce the Self-Tuning Actor-Critic (STAC) agent. STAC self-tunes a variant of the V-trace operator that we call leaky V-trace (in addition to the previous five meta parameters), motivated by the study of nonlinear activations in Deep Learning \citep{xu2015empirical}. Leaky V-trace uses a leaky rectifier \citep{maas2013rectifier} to truncate the importance sampling weights, where a differentiable parameter controls the leakiness. Moreover, it provides smoother gradients and prevents the unit from getting saturated. 

Before we introduce Leaky V-trace, let us first recall how the off-policy trade-offs are represented in V-trace using the coefficients $\bar \rho, \bar c$. The weight $\rho_t = \min (\bar \rho, \frac{\pi(a_t|x_t)}{\mu(a_t|x_t)})$ appears in the definition of the temporal difference $\delta_tV$ and defines the fixed point of this update rule. The fixed point of this update is the value function $V^{\pi_{\bar \rho}}$ of the policy $\pi_{\bar \rho}$  that is somewhere between the behavior policy $\mu$ and the target policy $\pi$ controlled by the hyperparameter $\bar \rho,$
$$
    \pi_{\bar \rho} = \frac{\min \left(\bar \rho \mu(a|x), \pi(a|x)\right)}{\sum_b \min \left(\bar \rho \mu(b|x), \pi(b|x)\right)}.
$$

The product of the weights $c_s, ..., c_{t-1}$ in \cref{eq:vtrace} measures how much a temporal difference $\delta_t V$ observed at time $t$ impacts the update of the value function. The truncation level $\bar c$ is used to control the speed of convergence by trading off the update variance for a larger contraction rate. The variance associated with the update rule is reduced relative to importance-weighted returns by clipping the importance weights. On the other hand, the clipping of the importance weights effectively cuts the traces in the update, resulting in the update placing less weight on later TD errors and worsening the contraction rate of the corresponding operator. Following this interpretation of the off policy coefficients, we propose a variation of V-trace which we call \textit{leaky V-trace} with parameters $\alpha_\rho \ge \alpha_c,$
\begin{align}
    \label{eq:leaky_vtrace}
    \text{IS}_t &= \frac{\pi(a_t|x_t)}{\mu(a_t|x_t)}, \enspace \rho _t =  \alpha_\rho \min \big(\bar \rho, \text{IS}_t\big) + (1- \alpha_\rho)\text{IS}_t, \enspace c_i = \lambda \big( \alpha_c  \min \big(\bar c, \text{IS}_t\big)  +(1-\alpha_c )\text{IS}_t\big), \nonumber \\ v_s &= V(x_s)+ \sum\nolimits_{t=s}^{s+n-1} \gamma ^{t-s} \left(\Pi ^{t-1}_{i=s}c_i \right)\delta_t V, \enspace
    \delta _t V = \rho_t (r_t + \gamma V(x_{t+1})-V(x_t)).
\end{align}

We highlight that for $\alpha_\rho=1, \alpha_c=1,$ Leaky V-trace is exactly equivalent to V-trace, while for $\alpha_\rho=0, \alpha_c=0,$ it is equivalent to canonical importance sampling. For other values, we get a mixture of the truncated and not-truncated importance sampling weights.

\cref{thm:leaky} below suggests that Leaky V-trace is a contraction mapping and that the value function that it will converge to is given by $V^{\pi_{\bar \rho, \alpha_\rho}},$ where 
$$
     \pi_{\bar \rho, \alpha_\rho} = \frac{\alpha_\rho\min \left(\bar \rho \mu(a|x), \pi(a|x)\right) +(1-\alpha_\rho)\pi(a|x)}{\alpha_\rho \sum_b\min \left(\bar \rho \mu(b|x), \pi(b|x)\right)+1-\alpha_\rho},
$$
is a policy that mixes (and then re-normalizes) the target policy with the V-trace policy. %\footnote{Note that $\alpha-$trace \citep{rowland2019adaptive}, another adaptive algorithm for off-policy learning, mixes the V-trace policy with the behavior policy; Leaky V-trace mixes it with the target policy.} 
We provide a formal statement of \cref{thm:leaky}, and detailed proof in the supplementary material (\cref{sec:proof}). 
\begin{theorem}
\label{thm:leaky}
     The leaky V-trace operator defined by \cref{eq:leaky_vtrace} is a contraction operator, and it converges to the value function of the policy defined above. 
\end{theorem}
Similar to $\bar \rho,$ the new parameter $\alpha_\rho$ controls the fixed point of the update rule and defines a value function that interpolates between the value function of the target policy $\pi$ and the behavior policy $\mu.$  Specifically, the parameter $\alpha_c$ allows the importance weights to "leak back" creating the opposite effect to clipping. 
Since \cref{thm:leaky} requires us to have $\alpha_\rho \ge \alpha_c,$ our main STAC implementation parametrises the loss with a single parameter $\alpha=\alpha_\rho=\alpha_c$. In addition, we also experimented with a version of STAC that learns both $\alpha_\rho$ and $\alpha_c$. This variation of STAC learns the rule $\alpha_\rho \ge \alpha_c$ on its own (see \cref{fig:discovery}). Note that low values of $\alpha_c$ lead to importance sampling, which is high contraction but high variance. On the other hand, high values of $\alpha_c$ lead to V-trace, which is lower contraction and lower variance than importance sampling. Thus exposing $\alpha_c$ to meta-learning enables STAC to control the contraction/variance trade-off directly. 
 
In summary, the \textbf{metaparameters} for STAC are $ \{ \gamma, \lambda, g_v, g_p, g_e ,\alpha\}.$ To keep things simple, when using Leaky V-trace we make two simplifications w.r.t the \textbf{hyperparameters}. First, we use V-trace to initialise Leaky V-trace, i.e., we initialise $\alpha=1.$ Second, we fix the outer loss to be V-trace, i.e. we set $\alpha^{\text{outer}}=1.$

\subsection{STAC with auxiliary tasks (STACX)}
Next, we introduce a new agent, that extends STAC with auxiliary policies, value functions, and corresponding auxiliary loss functions. Auxiliary tasks have proven to be an effective solution to learning useful representations from limited amounts of data. We observe that each set of meta-parameters induces a separate inner loss function, which can be thought of as an auxiliary loss. To meta-learn auxiliary tasks, STACX self-tunes additional sets of meta-parameters, independently of the main head, but via the same meta-gradient mechanism. The novelty here comes from STACX's ability to discover the auxiliary tasks most useful to it. E.g., the discount factors of these auxiliary losses allow STACX to reason about multiple horizons.

STACX's architecture has a shared representation layer $\theta_{\text{shared}}$, from which it splits into $n$ different heads (\cref{fig:arch}). For the shared representation layer we use the deep residual net from  \citep{espeholt2018impala}. Each head has a policy and a corresponding value function that are represented using a $2$ layered MLP with parameters $\{ \theta_i \}_{i=1}^n$. Each one of these heads is trained in the inner loop to minimize a loss function $L(\theta_i;\eta_i),$  parametrized by its own set of metaparameters $\eta_i$.

The STACX agent policy is defined as the policy of a specific head ($i=1$). We considered two more variations that allow the other heads to act. Both did not work well, and we provide more details in the supplementary (\cref{sec:hyper}).  The hyperparameters $\{\eta_i \}_{i=1}^n$ are trained in the outer loop to improve the performance of this single head. Thus, the role of the auxiliary heads is to act as auxiliary tasks \citep{jaderberg2016reinforcement} and improve the shared representation $\theta_{\text{shared}}$. Finally, notice that each head has its own policy $\pi^i$, but the behavior policy is fixed to be $\pi^1$. Thus, to optimize the auxiliary heads, we use (Leaky) V-trace for off-policy corrections.

\begin{wrapfigure}{r}{0.5\textwidth}
\includegraphics[width=\linewidth, trim=150 0 100 0, clip]{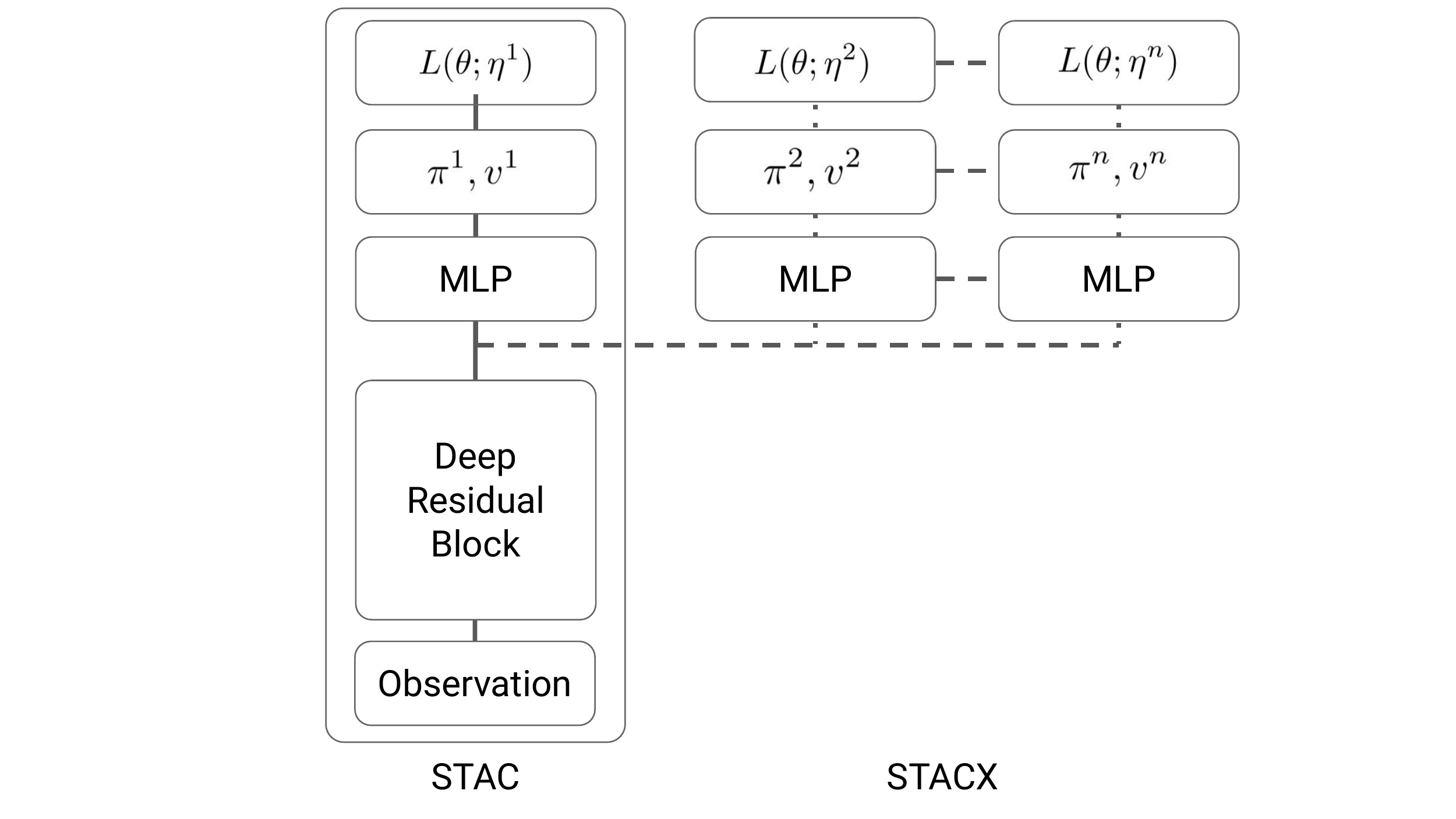}
\label{fig:arch}
\caption{Block diagrams of STAC and STACX.}
\end{wrapfigure} 

The \textbf{metaparameters}  for STACX are $ \{ \gamma^i, \lambda^i, g^i_v, g^i_p, g^i_e ,\alpha^i \}_{i=1}^3.$ Since the outer loss is defined only w.r.t head $1,$ introducing the auxiliary tasks into STACX does not require new hyperparameters for the outer loss. In addition, we use the same initialization values for all the auxiliary tasks. Thus, STACX has the same \textbf{hyperparameters} as STAC.

\textbf{Summary.}
In this Section, we showed how embracing self-tuning via metagradients enables us to introduce novel ideas into our agent. We augmented our agent with a parameterized Leaky V-trace operator and with \emph{self-tuned} auxiliary loss functions. We did not have to tune these new hyperparameters because we relied on metagradients to self-tune them. We emphasize here that STACX is not a fully parameter-free algorithm. Nevertheless, we argue that STACX requires the same hyperparameter tuning as IMPALA, since we use default values for the new hyperparameters. We further evaluated these design principles in \cref{fig:rob}.

\section{Experiments}

\subsection{Atari Experiments.}
We begin the empirical evaluation of our algorithm in the Arcade Learning Environment \citep[ALE]{bellemare2013arcade}. To be consistent with prior work, we use the same ALE setup that was used in \citep{espeholt2018impala,xu2018meta}; in particular, the frames are down-scaled and grayscaled.

\begin{figure}[b]
    \vspace{-0.5cm}
    \centering
    \subfigure[Average learning curves with $0.5 \cdot$ std confidence intervals.]{\label{fig:lcurve}\includegraphics[width=0.6\linewidth ,trim=0 0 0 50, clip]{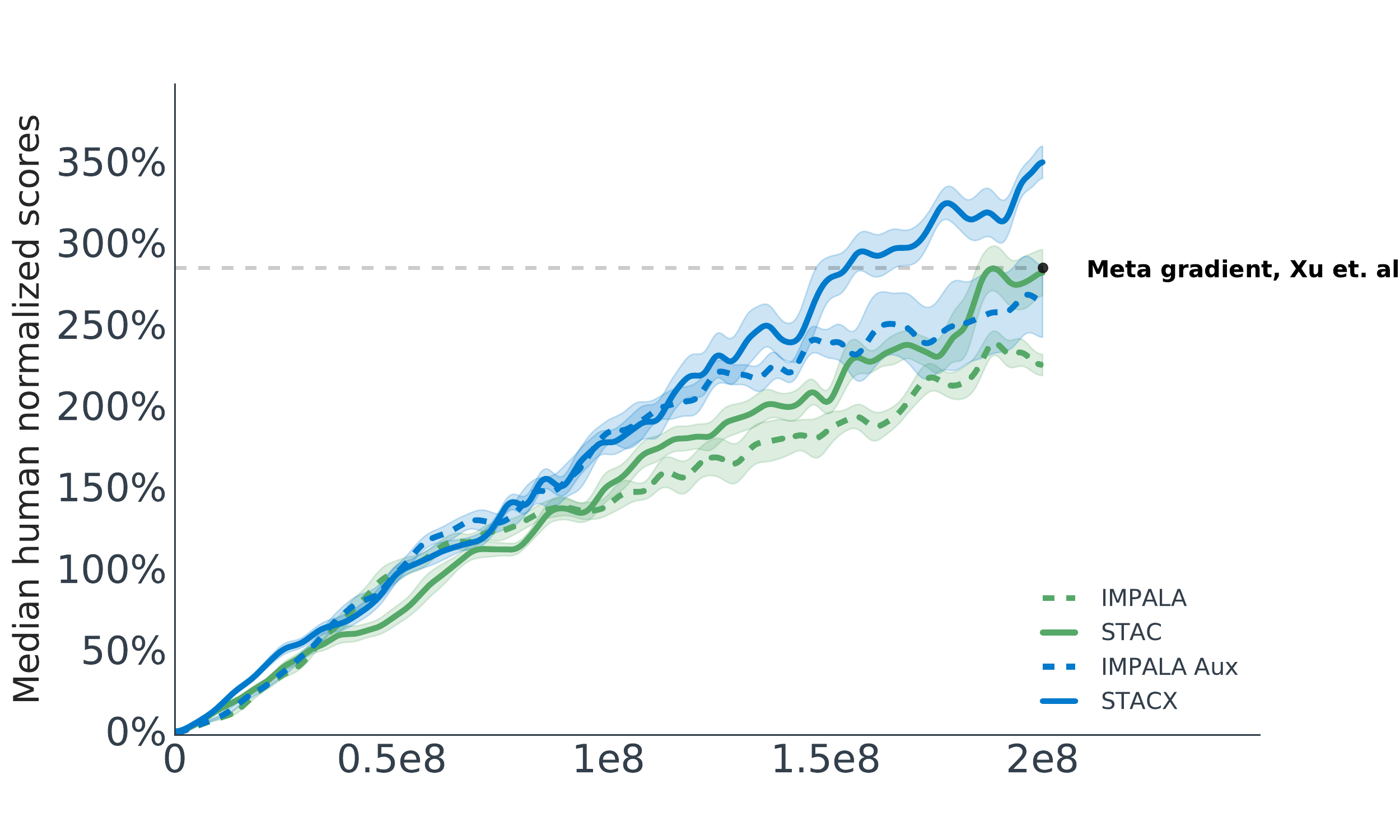}}
    \subfigure[Ablative studies of STAC (green) and STACX (blue) alongside baselines (red).]{\label{fig:generality}\includegraphics[width=0.39\linewidth]{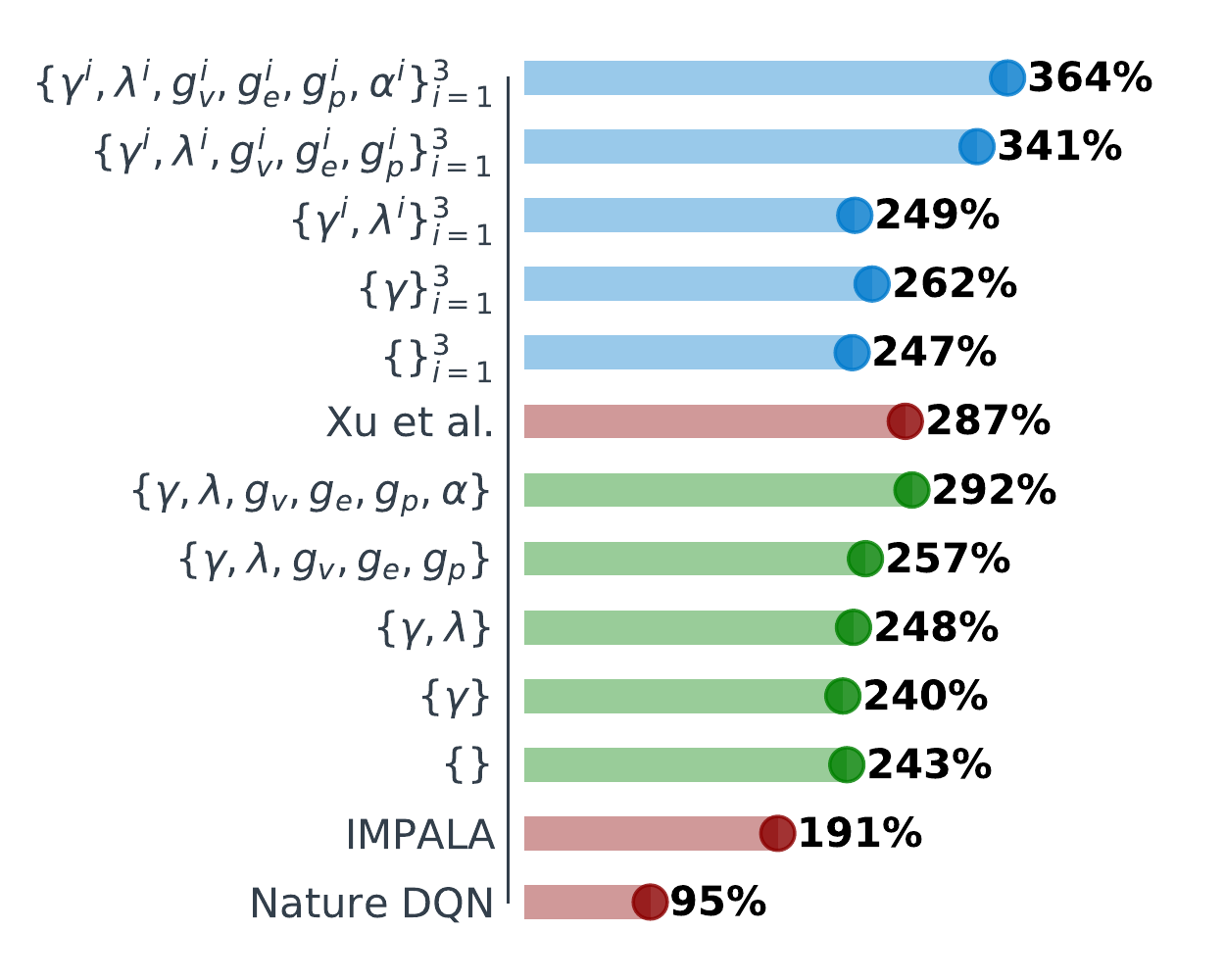}}
    \caption{Median normalized scores in 57 Atari games. Average over three seeds, 200M frames.}
\end{figure}

\cref{fig:lcurve} presents the normalized median scores during training, computed in the following manner: for each Atari game, we compute the human normalized score after 200M frames of training and average this over three seeds. We then report the overall median score over the $57$ Atari domains for four variations of our algorithm STACX (blue, solid), STAC (green, solid), IMPALA with fixed auxiliary tasks (blue, dashed), and IMPALA (green, dashed). Inspecting \cref{fig:lcurve} we observe two trends: using self-tuning improves the performance with/out auxiliary tasks (solid vs. dashed lines), and using auxiliary tasks improves the performance with/out self-tuning (blue vs. green lines). In the supplementary (\cref{sec:rel_atari}), we report the relative improvement over IMPALA in individual games. 

STACX outperforms all other agents in this experiment, achieving a median score of $364\%$, a new state of the art result in the ALE benchmark for training online model-free agents for 200M frames. In fact, there are only two agents that reported better performance after 200M frames: LASER \citep{schmitt2019off} achieved a normalized median score of $431\%$ and MuZero \citep{schrittwieser2019mastering} achieved $731\%$. These papers propose algorithmic modifications that are orthogonal to our approach and can be combined in future work; LASER combines IMPALA with a uniform large-scale experience replay; MuZero uses replay and a tree-based search with a learned model.

In \cref{fig:generality}, we perform an ablative study of our approach by training different variations of STAC (green) and STACX (blue). For each bar, we report the subset of metaparameters that are being self-tuned in this ablative study. The bottom bar for each color with $\left\{ \right\}$ corresponds to not using self-tuning at all (IMPALA w/o auxiliary tasks), and the topmost color corresponds to self-tuning all the metaparameters (as reported in \cref{fig:lcurve}). In between, we report results for tuning only subsets of the metaparameters. For example, $\eta = \{\gamma \}$ corresponds to self-tuning a single loss function where only $\gamma$ is self-tuned. When we do not self-tune a hyperparameter, its value is fixed to its corresponding value in the outer loss. For example, in all the ablative studies besides the two topmost bars, we do not self-tune $\alpha,$ which means that we use V-trace instead (fix $\alpha=1$).
Finally, in red, we report results from different baselines as a point of reference (in this case, IMPALA is using $\gamma=0.99$), and our variation of IMPALA (green, bottom) with $\gamma=0.995$ indeed achieves higher score as was reported in \citep{xu2018meta}. We also note that the metagradient agent of \citet{xu2018meta} achieved higher performance than our variation of STAC that is only learning $\eta = \left\{ \gamma, \lambda\right\}.$ We further discuss this in the supplementary (\cref{sec:rep_xu}). 

Inspecting \cref{fig:generality} we observe that the performance of STAC and STACX consistently improves as they self-tune more metaparameters. These metaparameters control different trade-offs in reinforcement learning: discount factor controls the effective horizon, loss coefficients affect learning rates, the Leaky V-trace coefficient controls the variance-contraction-bias trade-off in off-policy RL.

% \textbf{todo: add rgb results}.

\subsection{DM control suite}
To further examine the generality of STACX we conducted a set of experiments in the DM control suite \citep{tassa2018deepmind}. We considered three setups: (a) learning from feature observations, (b) learning from pixel observations, and (c) the real-world RL challenge \citep[RWRL]{dulac2020empirical}. The latter introduces a set of challenges (inspired by real-world scenarios) on top of existing control domains: delayed actions, observations and rewards, action repetition, added noise to the actions, stuck/dropped sensors, perturbations, and increased state dimensions. These challenges are combined in 3 difficulty levels (easy, medium, and hard) for humanoid, walker, quadruped, and cartpole. Scores are normalized to $[0,1000]$ by the environment \citep{tassa2018deepmind}.

We use the same algorithm and similar hyperparameters to the ones we use in the Atari experiments. For most of the hyperparameters (and in particular, those that are relevant to STACX) we use the same values as we used in the Atari experiments (e.g., $g_v,g_p,g_e$); others, like learning rate and discount factor, were re-tuned for the control domains (but remain fixed across all three setups). The exact details can be found in the supplementary (\cref{sec:hyper}). 
For continuous actions, our network outputs two variables per action dimension that correspond to the mean and the standard deviation of a squashed Gaussian distribution \citep{haarnoja2018soft}. The squashing refers to applying a tanh activation on the samples of the Gaussian, resulting in bounded actions. In addition, instead of using entropy regularization, we use a KL to standard Gaussian.

We emphasize here that while online actor-critic algorithms (IMPALA, A3C) do not achieve SOTA results in DM control, the results we present for IMPALA in \cref{fig:features} are consistent with the A3C results in \citep{tassa2018deepmind}. The goal of these experiments is to measure the relative improvement from self-tuning. In \cref{fig:dmcontrol}, we average the results across suite domains and across three seeds. Standard deviation error bars w.r.t the seeds are reported in shaded areas. In the supplementary (\cref{sec:dm_control}) we provide domain-specific learning curves. 

Inspecting \cref{fig:dmcontrol} we observe two trends. \textbf{First}, using self-tuning improves performance (solid vs. dashed lines) in all three suites, w/o using the auxiliary tasks. \textbf{Second}, the auxiliary tasks improve performance when learning from pixels (\cref{fig:pixels}), which is consistent with the results in Atari. When learning from features (\cref{fig:features}, \cref{fig:rwrl}), we observe that IMPALA performs better without auxiliary tasks. This is reasonable, as there is less need for strong representation learning in this case. Nevertheless, STACX performs better than IMPALA as it can self-tune the loss coefficients of the auxiliary tasks to low values. Since this takes time, STACX performs worse than STAC. 

Similar to the A3C baseline (using features), all of our agents were not able to solve the more challenging control domains (e.g., humanoid). Nevertheless, by using self-tuning, STAC, and STACX significantly outperformed the IMPALA baselines in many of the control domains. In the RWRL challenge, they even outperform strong baselines like D4PG and DMPO in the average score. Moreover, STAC was able to solve (to achieve an almost perfect score) two RWRL domains (quadruped.easy, cartpole.easy), making a new SOTA in these domains. 

\begin{figure}[h]
    \centering
    \subfigure[Feature observations]{\label{fig:features}\includegraphics[width=0.32\linewidth]{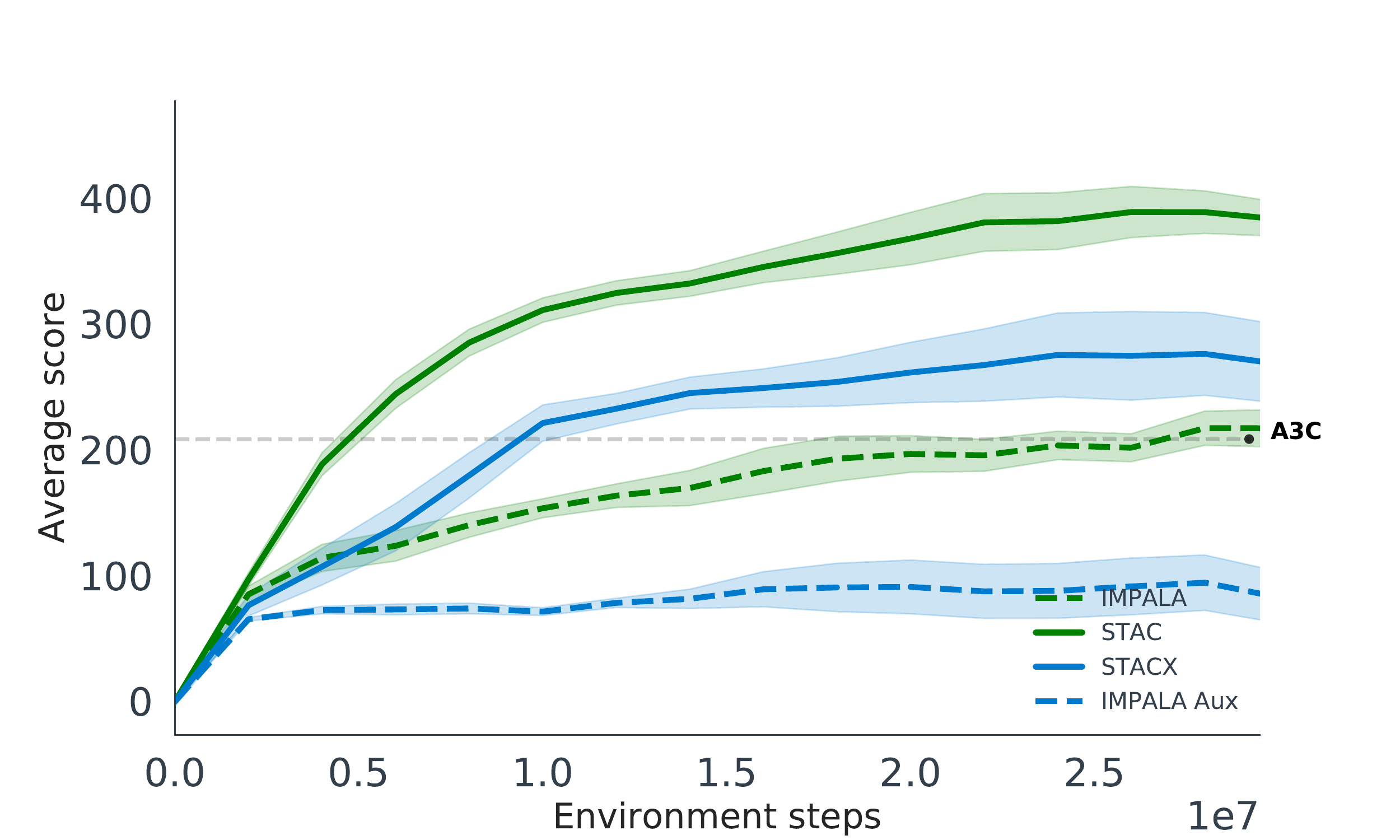}}
    \subfigure[Pixel observations]{\label{fig:pixels}\includegraphics[width=0.32\linewidth]{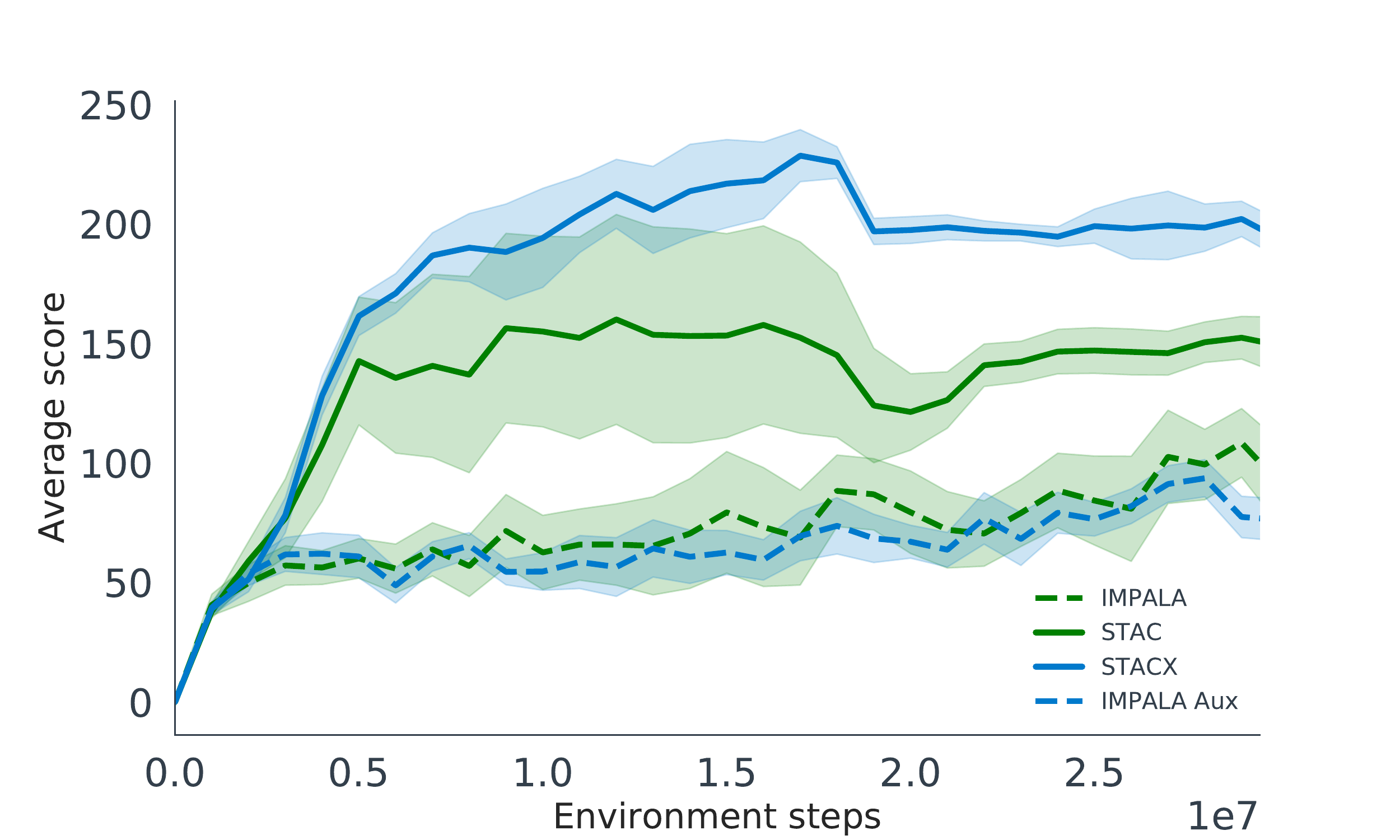}}
    \subfigure[Real world RL]{\label{fig:rwrl}\includegraphics[width=0.32\linewidth]{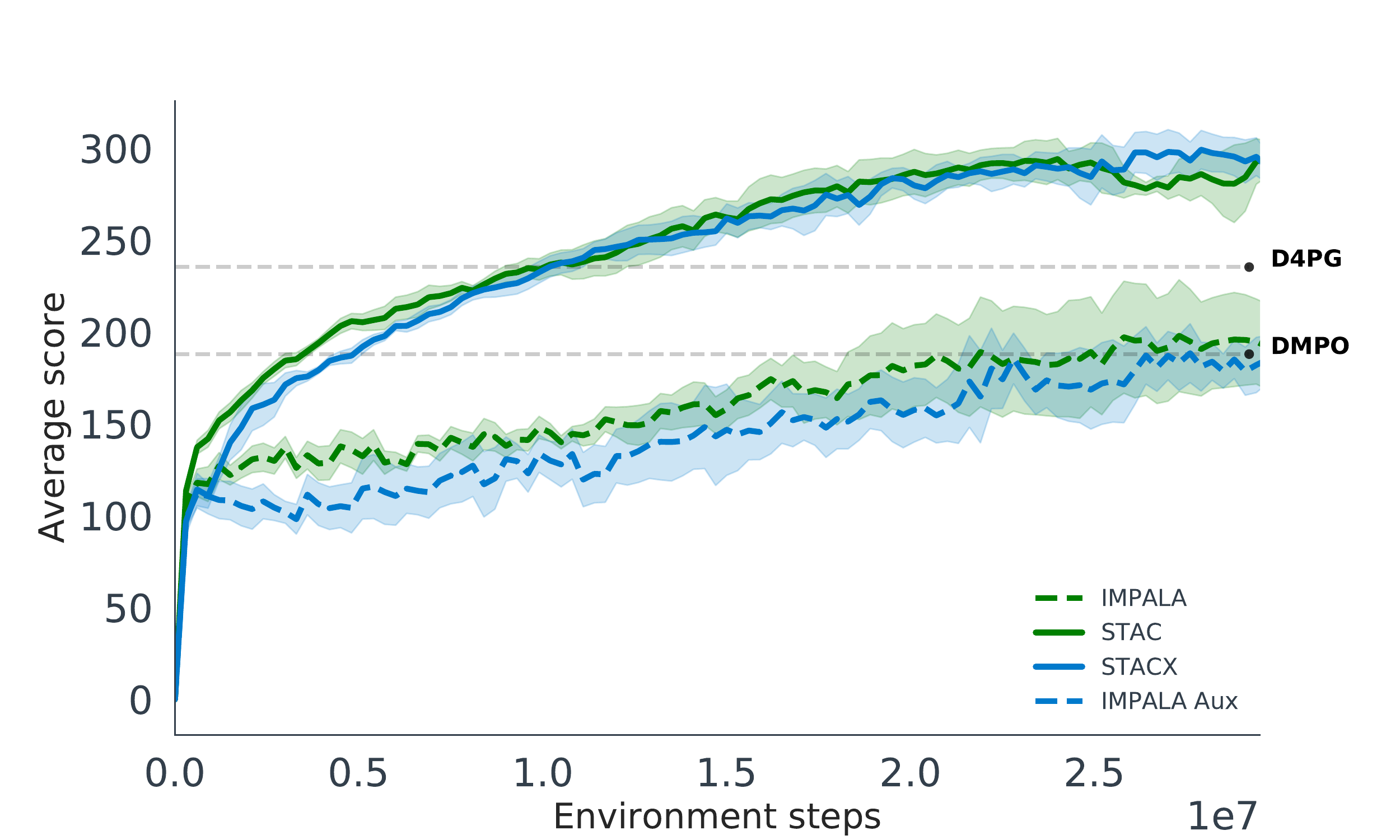}}
    \caption{Aggregated results in DM Control of STACX, STAC and IMPALA with/out Auxiliary tasks. In dashed lines, we report the aggregated results of baselines at the end of training; A3C as reported in \citep{tassa2018deepmind}, DMPO and D4PG as reported in \citep{dulac2020empirical}.}
    \label{fig:dmcontrol}
\end{figure}

\subsection{Analysis}
To better understand the behavior of the proposed method, we chose one domain, Atari, and performed some additional experiments.
We begin by investigating the \textbf{robustness} of STACX to its hyperparameters. First, we consider the hyperparameters of the \textbf{outer loss}, and compare the robustness of STACX with that of IMPALA. For each hyperparameter ($\gamma, g_v$) we select $5$ perturbations. For STACX we perturb the hyperparameter in the outer loss ($\gamma^{\text{outer}}, g^{\text{outer}}_v$) and for IMPALA we perturb the corresponding hyperparameter ($\gamma, g_v$). We randomly selected $5$ Atari games and presented the mean and standard deviation across $3$ random seeds after $200$M frames. 

\cref{fig:rob_gamma} and \cref{fig:rob_baseline} present the results for the discount factor $\gamma$ and for $g_v$ respectively. We can see that overall, STACX performs better than IMPALA (in $72\%$ and $80\%$ of the setups, respectively). This is perhaps not surprising because we have already seen that STACX outperforms IMPALA in Atari, but now we observe this over a wider range of hyperparameters. 
In addition, we can see that in specific games, there are specific hyperparameter values that result in lower performance. In particular, in James Bond and Chopper Command (the two topmost rows), we observe lower performance when decreasing the discount factor $\gamma$ and when decreasing $g_v$. While the performance of both STACX and IMPALA deteriorates in these experiments, the effect on STACX is less severe. 

In \cref{fig:rob_init} we investigate the robustness of STACX to the \textbf{initialization} of the metaparameters. Since IMPALA does not have this hyperparameter, we only investigate its effect on STACX. We selected five different initialization values (all close to $1,$, so the inner loss is close to the outer loss) and fixed all the other hyperparameters (e.g., the outer loss). Inspecting \cref{fig:rob_init}, we can see that the performance of STACX does not change too much when we change the value of the initializations, both in the case where we perturb the initializations of all the meta parameters (top), and only the discount (bottom). These observations confirm that our design choice to arbitrary initializing all the meta parameters to $0.99$ is sufficient, and there is no need to tune this new hyperparameter. 

\begin{figure}[h]
    \centering
    \subfigure[Discount factor.]{\label{fig:rob_gamma}\includegraphics[width=0.34\linewidth]{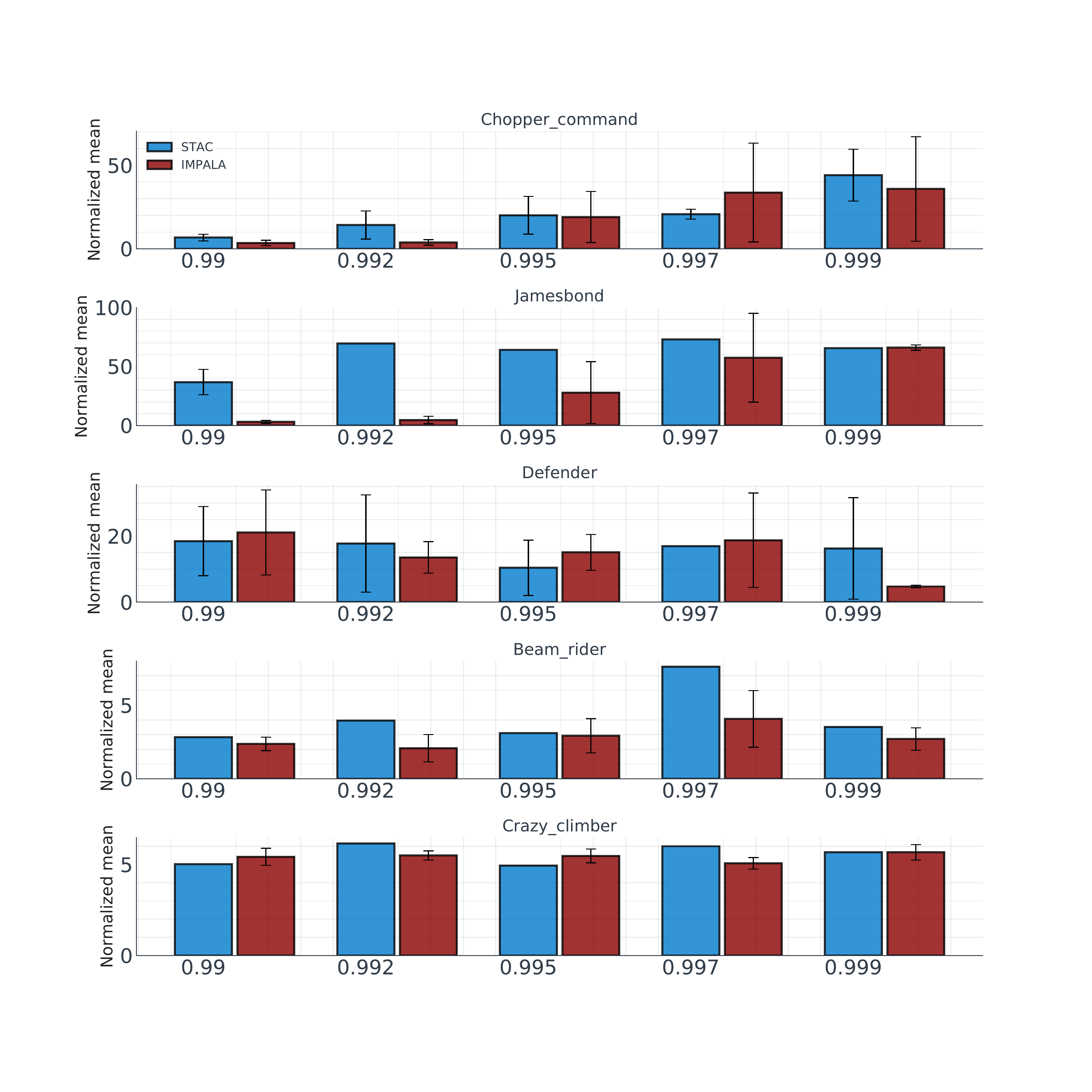}}
    \subfigure[Critic weight $g_v$.]{\label{fig:rob_baseline}\includegraphics[width=0.34\linewidth]{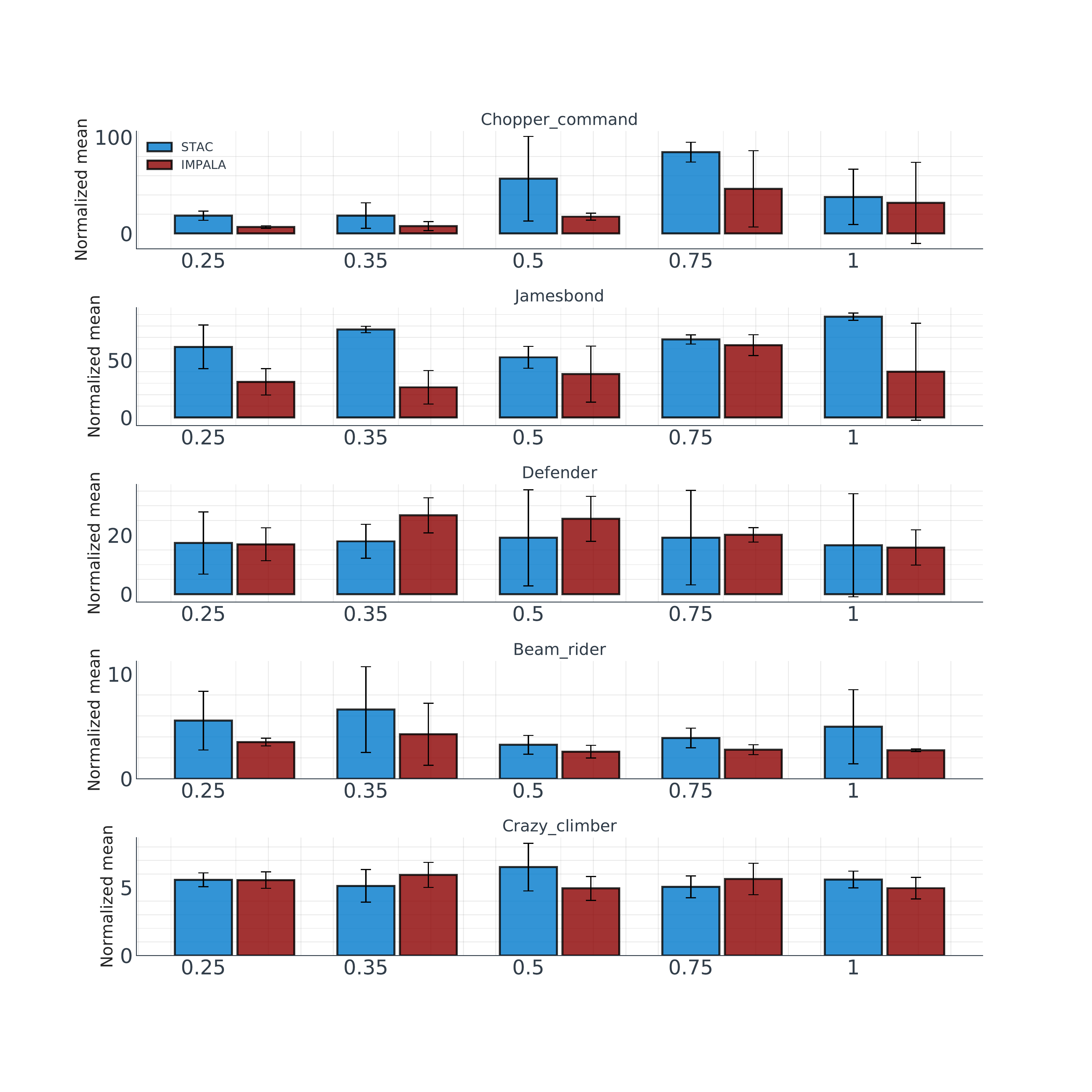}}
    \subfigure[Metaparameters initialisation.]{\label{fig:rob_init}\includegraphics[width=0.3\linewidth]{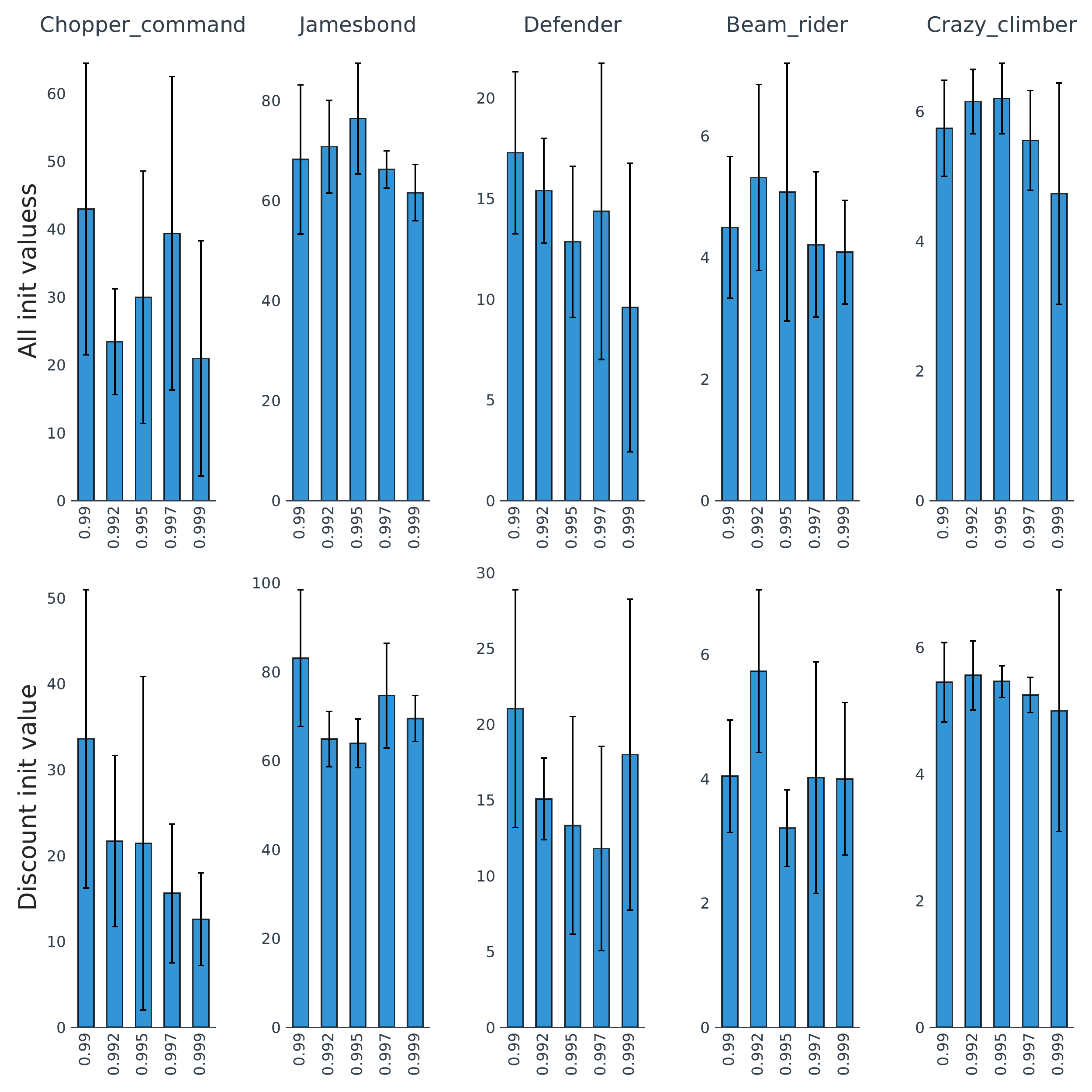}}
    \caption{Robustness results in Atari. Mean and confidence intervals (over 3 seeds), after 200M frames of training. \textbf{\cref{fig:rob_gamma} and \cref{fig:rob_baseline}:} blue bars correspond to STACX and red bars to IMPALA. Rows correspond to specific Atari games, and columns to the value of the hyper parameter in the outer loss ($\gamma, g_v$). We observe that STACX is better than IMPALA in 72$\%$ of the runs (\cref{fig:rob_gamma}) and in 80$\%$ of the runs in \cref{fig:rob_baseline}. \textbf{\cref{fig:rob_init}:} Robustness to the initialisation of the metaparameters. Columns correspond to different games. \textbf{Bottom:} perturbing $\gamma^{\text{init}} \in \{0.99,0.992,0.995,0.997,0.999 \}$. \textbf{Top:} perturbing all the meta parameter initialisations. I.e., setting all the hyperparamters  $\{\gamma^{\text{init}},\lambda^{\text{init}}, g_v^{\text{init}}, g_e^{\text{init}}, g_p^{\text{init}}, \alpha^{\text{init}} \} _{i=1}^3 $ to a single fixed value in $\{0.99,0.992,0.995,0.997,0.999 \}$.  }
    \label{fig:rob}
\end{figure}

\textbf{Adaptivity.}
\label{sec:adap}
In \cref{fig:adaptivity_jamesbond} we visualize the metaparameters of STACX during training. The metaparameters associated with the policy head (head number 1) are in blue, and the auxiliary heads (2 and 3) are in orange and magenta. We present the values of the metaparameters used in the inner loss, i.e., after we apply a sigmoid activation. But to have a single scale for all the metaparameters ($\eta \in [0,1]$), we present the loss coefficients $g_e,g_v,g_p$ without scaling them by the respective value in the outer loss. For example, the value of the entropy weight $g_e$ that is presented in \cref{fig:adaptivity_jamesbond} is further multiplied by $g_e^{\text{outer}}=0.01$ when used in the inner loss. 
As there are many metaparameters, seeds, and games, we only present results on a single seed (chosen arbitrarily to 1) and a single game (James Bond). In the supplementary  we provide examples for all the games (\cref{sec:lc_atari}). 

Inspecting \cref{fig:adaptivity_jamesbond}, one can notice that the metaparameters are being adapted in a none monotonic manner that could not have been designed by hand. We highlight a few trends which are visible in \cref{fig:adaptivity_jamesbond} and we found to repeat across games (\cref{sec:lc_atari}). The metaparameters of the auxiliary heads are self-tuned to have relatively similar values but different than those of the main head. For example, the main head discount factor converges to the value in the outer loss (0.995). In contrast, the auxiliary heads' discount factors often change during training and get to lower values. Another observation is that the leaky V-trace parameter $\alpha$ remains close to 1 at the beginning of training, so it is quite similar to V-trace. Towards the end of the training, it self-tunes to lower values (closer to importance sampling), consistently across games. We emphasize that these observations imply that adaptivity happens in self-tuning agents. It does not imply that this adaptivity is directly helpful. We can only deduce this connection implicitly, i.e., we observe that self-tuning agents achieve higher performance and adapt their metaparameters through training.

In \cref{fig:discovery}, we experimented with a variation of STACX that self-tunes both $\alpha_\rho $ and $\alpha_c$ without imposing  $\alpha_\rho\ge\alpha_c$ (as \cref{thm:leaky} requires to guarantee contraction). Inspecting \cref{fig:discovery}, we can see that STACX self-tunes  $\alpha_\rho \ge \alpha_c$ in James Bond. In addition, we measured that across the $57$ games $\alpha_\rho \ge \alpha_c$ in $91.2\%$ of the time (averaged over time, seeds, and games), and that $\alpha_\rho \ge 0.99\alpha_c$ in $99.2\%$ of the time. In terms of performance, the median score ($353\%$) was slightly worse than STACX. A possible explanation is that while this variation allows more flexibility, it may also be less stable as the contraction is not guaranteed. 

In another variation we self-tuned $\alpha$ together with a single truncation parameter $\bar \rho = \bar c.$ This variation performed worse, achieving a median score of $301\%$, which may be explained by $\bar \rho$ not being differentiable, suffering from nonsmooth (sub) gradients and possibly saturated IS truncation levels.

\begin{figure}[h]
    \centering
    \subfigure[Adaptivity in james Bond.]{\label{fig:adaptivity_jamesbond}\includegraphics[width=0.45\linewidth]{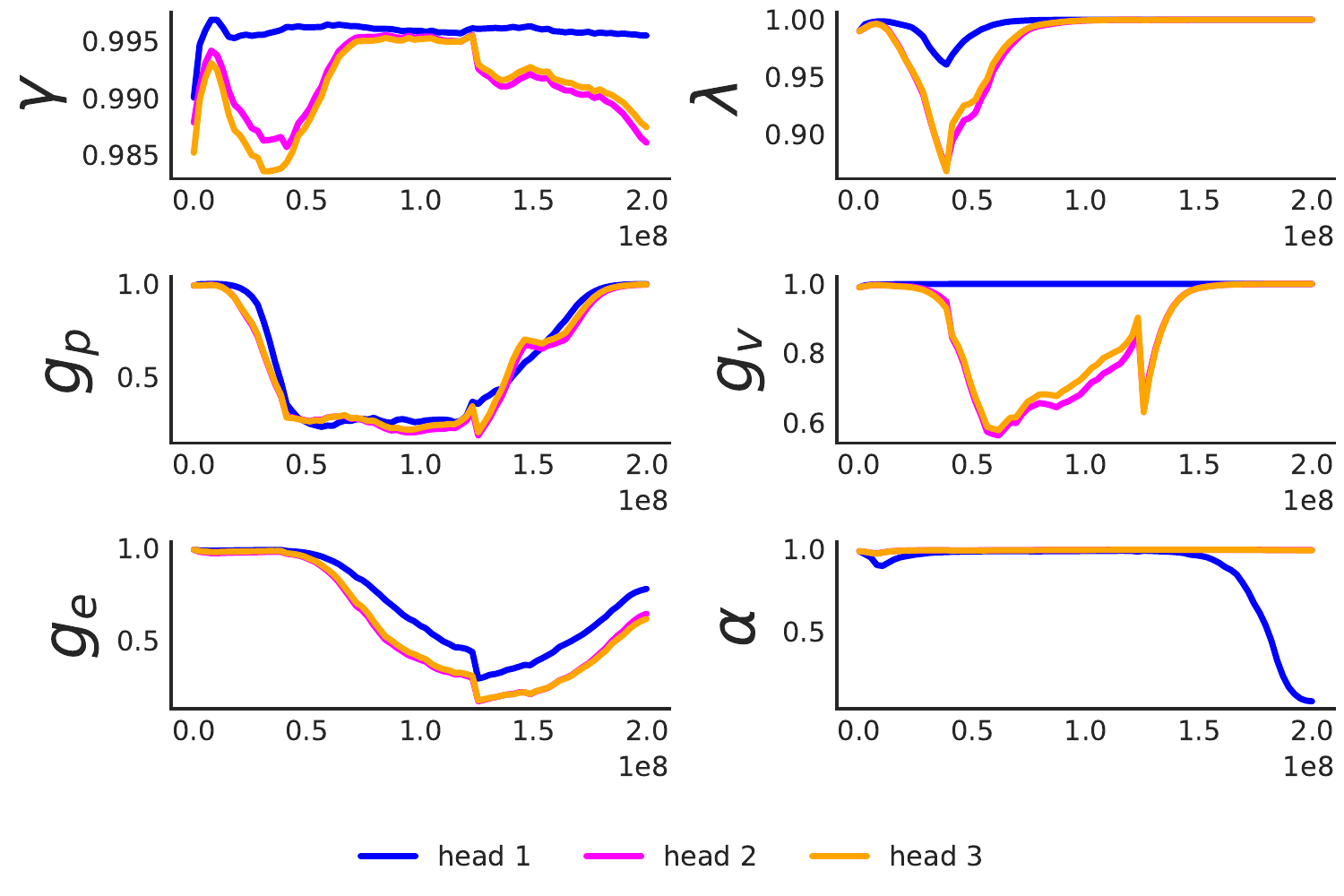}}
    \subfigure[Discovery of $\alpha_\rho \ge \alpha_c$.]{\label{fig:discovery}\includegraphics[width=0.52\linewidth]{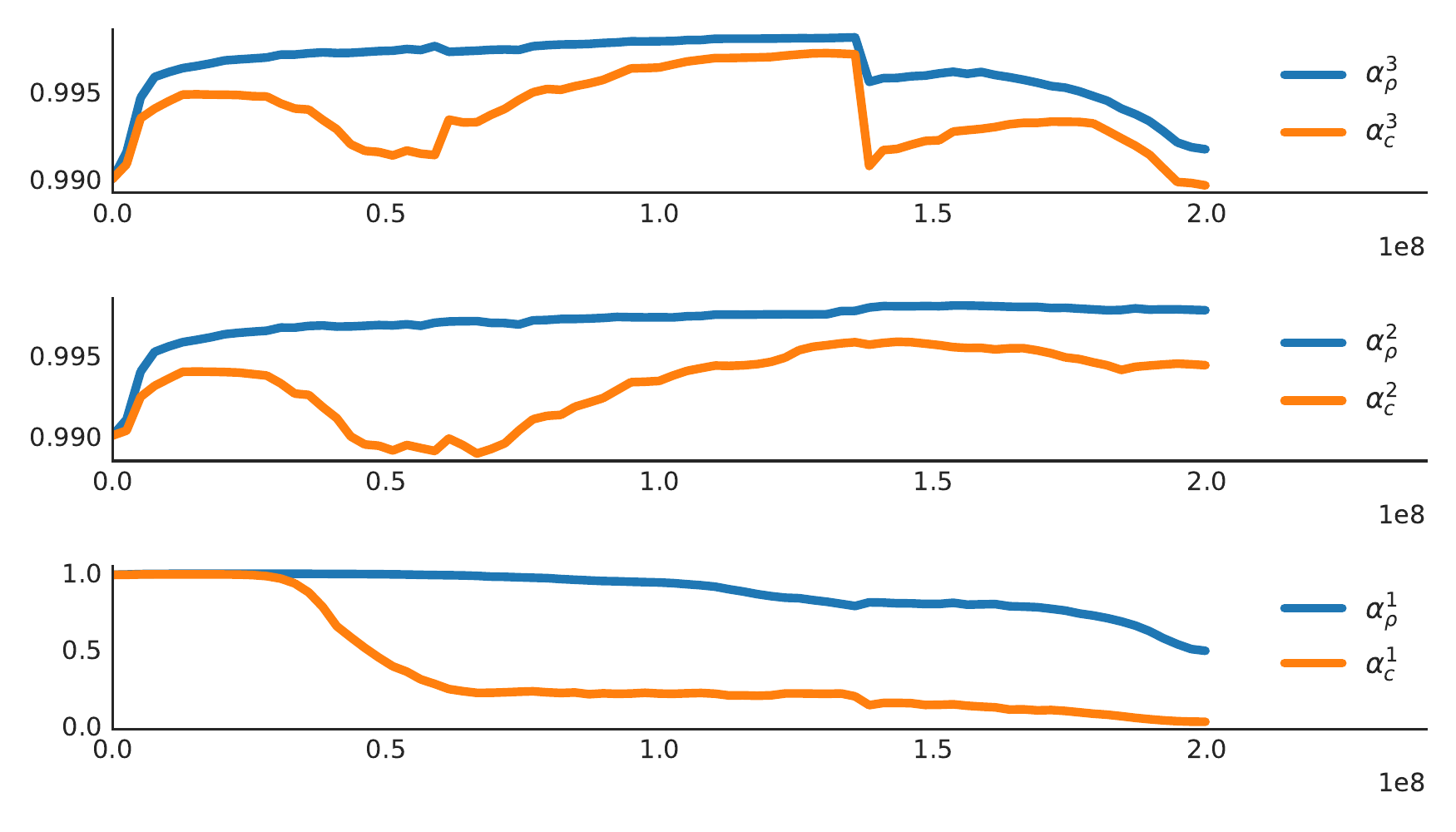}}
\end{figure}

\section{Summary}
In this work, we demonstrated that it is feasible to self-tune all the differentiable hyperparameters in an actor-critic loss function. We presented STAC and STACX, actor-critic algorithms that self-tune a large number of hyperparameters of different nature (controlling important trade-offs in a reinforcement learning agent) online, within a single lifetime. We showed that these agents' performance improves as they self-tune more hyperparameters. In addition, the algorithms are computationally efficient and robust to their hyperparameters. Despite being an online algorithm, STACX achieved very high results in the ALE and the RWRL domains. 

We plan to extend STACX with Experience Replay to make it more data-efficient in future work. By embracing self-tuning via metagradients, we were able to introduce these novel ideas into our agent, without having to tune their new hyperparameters. However, we emphasize that STACX is not a fully parameter-free algorithm; we hope to investigate further how to make STACX less dependent on the hyperparameters of the outer loss in future work.

\clearpage
\section{Broader Impact}
The last decade has seen significant improvements in Deep Reinforcement Learning algorithms. To make these algorithms more general, it became a common practice in the DRL community to measure the performance of a single DRL algorithm by evaluating it in a diverse set of environments, where at the same time, it must use a single set of hyperparameters. That way, it is less likely to overfit the agent's hyperparameters to specific domains, and more general properties can be discovered. These principles are reflected in popular DRL benchmarks like the ALE and the DM control suite. 

In this paper, we focus on exactly that goal and design a self-tuning RL agent that performs well across a diverse set of environments. Our agent starts with a global loss function that is shared across the environments in each benchmark. But then, it has the flexibility to self-tune this loss function, separately in each domain. Moreover, it can adapt its loss function within a single lifetime to account for inherent non-stationarities in RL algorithms - exploration vs. exploitation, changing data distribution, and degree of off-policy. 

While using meta-learning to tune hyperparameters is not new, we believe that we have made significant progress that will convince many people in the DRL community to use metagradients. We demonstrated that our agent performs significantly better than the baseline algorithm in four benchmarks. The relative improvement is much more significant than in previous metagradient papers and is demonstrated across a wider range of environments. While each of these benchmarks is diverse on its own,  together, they give even more significant evidence to our approach's generality. 

Furthermore, we show that it's possible to self-tune tenfold more metaparameters from different types. We also showed that we gain improvement from self-tuning various subsets of the meta parameters, and that performance kept improving as we self-tuned more metaparameters. Finally, we have demonstrated how embracing self-tuning can help to introduce new concepts (leaky V-trace and parameterized auxiliary tasks) to RL algorithms without needing tuning.  

\section{Acknowledgements}
We would like to thank Adam White and Doina Precup for their comments, Manuel Kroiss for implementing the computing infrastructure described in \cref{compute} and Thomas Degris for the environments infrastructure design.

The author(s) received no specific funding for this work.
\clearpage

\bibliography{paper_bib}
\bibliographystyle{icml2020}

\clearpage
\onecolumn

\section{Computing Infrastructure}
\label{compute}

We run our experiments using Sebulba, a podracer infrastructure \citep{hessel2021podracer} implemented in JAX \citep{jax2018github}, using the JAX libraries Haiku \citep{haiku2020github}, RLax \citep{rlax2020github} and Optax \citep{optax2020github}. The computing infrastructure is based on an actor-learner decomposition \citep{hessel2021podracer}, where multiple actors generate experience in parallel, and this experience is channelled into a learner.

It allows us to run experiments in two modalities. In the first modality, following \cite{espeholt2018impala}, the actors programs are distributed across multiple CPU machines, and both the stepping of environments and the network inference happens on CPU. The data generated by the actor programs is processed in batches by a single learner using a GPU. Alternatively, both the actors and learners are co-located on a single machine, where the host is equipped with 56 CPU cores and connected to 8 TPU cores \citep{TPUs}. To minimize the effect of Python's Global Interpreter Lock, each actor-thread interacts with a \textit{batched environment}; this is exposed to Python as a single special environment that takes a batch of actions and returns a batch of observations, but that behind the scenes steps each environment in the batch in C++. The actor threads share 2 of the 8 TPU cores (to perform inference on the network), and send batches of fixed size trajectories of length T to a queue. The learner threads takes these batches of trajectories and splits them across the remaining 6 TPU cores for computing the parameter update (these are averaged with an all reduce across the participating cores). Updated parameters are sent to the actor's TPU cores via a fast device to device channel, as soon as the new parameters are available. This minimal unit can be replicates across multiple hosts, each connected to its own 56 CPU cores and 8 TPU cores, in which case the learner updates are synced and averaged across all cores (again via fast device to device communication).

\section{Reproducibility}

\subsection{Open-sourcing}

We open-sourced our JAX implementation of the computation of leaky V-trace targets from trajectories of experience. The code can be found at www.github.com/anonymised-link.

\subsection{Pseudo code}
In this Section, we provide pseudo-code for STAC and STACX. To improve reproducibility, we also include information on where to stop gradients when computing the inner and outer losses. For that goal, we denote by $\text{sg}()$ the operation that is stopping the gradients from propagating through a function. 

We begin with \cref{alg:loss} that explains how to compute the IMPALA loss function. \cref{alg:loss} gets as inputs a set of trajectories $\Tau$, the agent parameters $\theta,$ parameters $\zeta$ that include hyperparameters and metaparameters, and a flag that indicates if it is an inner or an outer loss (see lines 23-26 to). 

\begin{algorithm}
\caption{IMPALA loss with V-trace}
\label{alg:loss}
\begin{algorithmic}[1]
    \State \textbf{Inputs}: $\Tau, \theta, \zeta, \text{inner loss}$ 
    \begin{itemize}
        \item \textbf{Data $\Tau$:} $m$ trajectories $\left\{ \tau_i \right\}_{i=1}^m$ of size $n,$ $\tau_i = \left\{ x^i_s, a^i_s, r^i_s, \mu(a^i_s|x^i_s) \right\}_{s=1}^n,$ where $\mu(a^i_s|x^i_s) $ is the probability assigned to $a^i_s$ in state $x^i_s$ by the behaviour policy $\mu(a|x).$
        \item  \textbf{Hyperparameters:} $\zeta = \left\{ g_v,g_e,g_p,\gamma,\lambda, \alpha, \bar c, \bar \rho \right\}.$
        \item  \textbf{Agent parameters:} $\theta,$ which define the policy $\pi_\theta(a|x)$ and value function $V_\theta (x)$.
        \item \textbf{Inner Loss:} Boolean flag representing inner/outer loss. 
    \end{itemize}
    \Statex
    \For{$i = 1,\ldots,m$} \Comment{Compute leaky V-trace targets}
        \For{$s = 1,\ldots,n$}
            \State Let $\text{IS}^i_s = \text{sg}(\frac{\pi(a^i_s|x^i_s)}{\mu(a^i_s|x^i_s)})$ \Comment{Importance sampling ratios} 
            \State Set  $c^i_s =  \left((1 - \alpha)  \text{IS}^i_s + \alpha \min (\bar c, \text{IS}^i_s)\right) \lambda$
            \State Set  $\rho^i_s =  (1 - \alpha)  \text{IS}^i_s + \alpha \min (\bar \rho, \text{IS}^i_s)$
            \State Set $\delta^i_s = \rho^i_s (r^i_{s+1} + \gamma \text{sg} (V_\theta(x^i_{s+1})) - \text{sg}(V_\theta(x^i_{s})))$ \Comment{One step td errors} 
        \EndFor
        \State Let $e^i_{n+1} = 0$
        \For{$s = n,\ldots,1$} \Comment{Compute n-step  td-errors backwards}
            \State $e_s^i = \delta^i_s + \gamma c_s^i e^i_{s+1}$
        \EndFor
        \For{$s = 1,\ldots,n$} \Comment{Compute leaky V-trace targets}
           \State $v_s^i = e_s^i + \text{sg}(V_\theta(x^i_{s}))$
        \EndFor
    \EndFor
    
    \Statex
    \State $L_{\text{Value}}(\theta) = g_v \cdot  \sum_{i=1,s=1}^{i=m,s=n-1} \left( v_s^i - V_\theta(x^i_{s})\right)^2$
    \State $L_{\text{Entropy}}(\theta) = -g_e \cdot  \sum_{i=1,s=1}^{i=m,s=n-1} \sum_{a} \pi(a_s^i|x_s^i)\log (\pi(a_s^i|x_s^i))$
    \If{Inner Loss}
    \State $L_{\text{Policy}}(\theta) = -g_p \cdot  \sum_{i=1,s=1}^{i=m,s=n-1} \rho_s^i \log (\pi(a_s^i|x_s^i)) \left( r_{s+1}^i + \gamma v_{s+1}^i - \text{sg}(V_\theta(x^i_{s}))\right)$
    \Else
    \State $L_{\text{Policy}}(\theta) = -g_p \cdot  \sum_{i=1,s=1}^{i=m,s=n-1} \rho_s^i \log (\pi(a_s^i|x_s^i)) \text{sg}\left( r_{s+1}^i + \gamma v_{s+1}^i - V_\theta(x^i_{s})\right)$
    \EndIf
    \State \textbf{Return} $L(\theta) = L_{\text{Value}}(\theta) + L_{\text{Policy}}(\theta)+ L_{\text{Entropy}}(\theta)$
\end{algorithmic}
\end{algorithm}

At each iteration $t$ \cref{alg:stac} calls \cref{alg:loss} to update the parameters $\theta_t$ and the metaparameters $\eta_t$ by differentiating the inner loss w.r.t $\theta$ and for the outer losses w.r.t $\eta$. In line 10, given the current metaparameters $\eta_t$, we apply a set of nonlinear transformations to compute the values of the hyperparameters of the inner loss. We apply a sigmoid activation on all the metaparameters, which ensures that they remain bounded, and multiply the loss coefficients by their corresponding values in the outer loss to guarantee that they are initialized from the same values.

\begin{algorithm}
\caption{Inner and outer loss}
\label{alg:stac}
\begin{algorithmic}[1]
\State Let $\text{Loss}(\Tau,\theta,\zeta, \text{inner loss})$ be a function that calls \cref{alg:loss}.
\State Let $\sigma(x)$  denote applying a sigmoid activation on $x$.
\State Let $\bar \rho =1, \bar c = 1$ be fixed truncation levels.
\State Denote the metaparameters at time $t$ by $\eta_t = \left\{\gamma, \lambda, \alpha, g_v, g_p, g_e \right\}$ 
\State Denote the hyperparameters by $\zeta = \left\{\gamma^{\text{outer}}, \lambda^{\text{outer}}, \alpha^{\text{outer}}, g_v^{\text{outer}}, g_p^{\text{outer}}, g_e^{\text{outer}}, \bar \rho, \bar c\right\}$
\State Let $\text{OPT}, \text{MetaOPT}$ be the optimizer and meta optimizer with their respective hyper parameters
\State Let $g^{\text{KL}}$ be the KL loss coefficient. 
\For{t  = 1 \ldots }
\State Collect trajectories $\Tau = \left\{\tau_i\right\}_{i=1}^m$ using the behaviour policy $\mu_t$
\State Set $\zeta(\eta_t) = \left\{\sigma(\gamma), \sigma(\lambda), \sigma(\alpha), \sigma(g_v)g_v^{\text{outer}}, \sigma(g_p)g_p^{\text{outer}}, \sigma(g_e)g_e^{\text{outer}}, \bar \rho, \bar c\right\}$ 
\State $\nabla{\theta_t}(\eta_t) = \frac{1}{P} \sum_{p=1}^P \nabla_\theta \left( \text{Loss} (\Tau, \theta_t^p,\zeta(\eta_t), \text{True})\right)$ \Comment{Gradient of the inner loss w.r.t $\theta$}
\State $\theta_{t+1}(\eta_t) = \ \text{OPT} (\theta_t, \nabla{\theta_t}(\eta_t))$
\State $\nabla{\eta_t} = \nabla_\eta \left( \text{Loss} (\Tau, \theta^1_{t+1}(\eta_t),\zeta, \text{False})\right)$
\Comment{Metagradient of the outer loss w.r.t $\eta$}
\State $\nabla{\eta_t}=\nabla{\eta_t} + g^{\text{KL}}\nabla_{\eta_t} \text{KL}(\pi_{\theta_{t+1}(\eta_t)}, \pi_{\theta_t}; \Tau)$
\Comment{Metagradient of the KL loss}

\State $\eta_{t+1} = \text{MetaOPT} (\eta_t, \nabla{\eta_t})$
\EndFor
\end{algorithmic}
\end{algorithm}

In line 11 the gradient of the inner loss is computed w.r.t $\theta$. Since the metaparameters define the loss, this gradient is parametrized by the metaparameters. We repeat the gradient computation for each head $p \in [1..P]$ where $P=1$ for STAC and $P=3$ for STACX; and the notation $\theta_t^p$ refers to the parameters that correspond to the p-th head. In practice, the heads share a single torso (see next subsection for details), so this step is computationally efficient. 

In line 12 we compute the update to the parameters $\theta$ by applying an optimizer update to parameters $\theta_t$ using the gradient $\nabla{\theta_t} (\eta_t)$ that was compute in line 11. In practice, we use the RMSProp optimizer for that. Since the metaparameters parametrized the gradient, so does $\theta_{t+1}.$

In line 13, we differentiate the outer loss, with fixed hyperparameters $\zeta$, w.r.t the metaparameters $\eta_t$. This is done by differentiating \cref{alg:loss} w.r.t $\eta_t$ through the parameterized parameters $\theta^1_{t+1}(\eta_t)$. Notice that when we compute the outer loss, we only take into account the loss of head 1. In line 14, we update the meta parameters by calling the meta optimizer (ADAM with default values). 

\subsection{Hyperparameters}
\label{sec:hyper}

\textbf{Architectures.}

\begin{table}[h]
\caption{Network architectures}
\begin{center}
\begin{tabular}{|l|l|l|l|}
    \hline
    Parameter & Atari & Control - pixels & Control - features  \\
    \hline 
    convolutions in block & (2, 2, 2) & (2, 2, 2) & - \\
    channels & (13, 32, 32) & (32, 32, 32) & -\\
    kernel sizes & (3, 3, 3) &  (3, 3, 3) & - \\
    kernel strides & (1, 1, 1) & (2, 2, 2) \\
    pool sizes & (3, 3, 3) & - & -\\
    pool strides & (2, 2, 2) & - & - \\
    mlp torso & - & - & (256, 256)\\
    lstm & - & 256 & 256 \\
    frame stacking & 4 & - & - \\
    head hiddens & 256 &  256 &  256 \\
    activation & Relu & Relu & Relu\\

    \hline
\end{tabular}
\end{center}
\label{table:hyperparameters2}
\end{table}

Our DNN architecture is composed of a shared torso, which then splits to different heads. We have a head for the policy and a head for the value function (multiplied by three for STACX). Each head is a two-layered MLP with 256 hidden units, where the output dimension corresponds to 1 for the vale function head. For the policy head, we have $|A|$ outputs that correspond to softmax logits when working with discrete actions (Atari),  and $2|A|$ outputs that correspond to the mean and standard deviation of a squashed Gaussian distribution with a diagonal covariance matrix in the case of continuous actions. We use ReLU activations on the outputs of all the layers besides the last layer. 

We use a softmax distribution for the policy and the entropy of this distribution for regularization for discrete actions. For continuous actions, we apply a tanh activation on the output of the mean. For the standard deviation, for output $y,$ the standard deviation is given by 

$$
\sigma (y) = \exp{ \sigma_{\text{min}} + 0.5 \cdot ( \sigma_{\text{max}} - \sigma_{\text{min}}) \cdot (\text{tanh}(y) + 1))}
$$

We then sample action from a Gaussian distribution with a diagonal covariance matrix $N(\mu,\sigma)$ and apply a tanh activation on the sample. These transformations guarantee that the sampled action is bounded. We then adjust the probability and log probabilities of the distribution by making a Jacobian correction \citep{haarnoja2018soft}. 

The \textbf{torso} of the network is defined per domain. When learning from features (RWRL and DM control), we use a two-layered MLP, with hidden layers specified in \cref{table:hyperparameters2}. For Atari and DM control from pixels, our network uses a convolution torso. The torso is composed of residual blocks. In each block there is a convolution layer, with stride, kernel size, channels specified per domain (Atari, Control from pixels) in \cref{table:hyperparameters2}, with an optional pooling layer following it. The convolution layer is followed by n - layers of convolutions (specified by blocks), with a skip contention. The output of these layers is of the same size of the input so they can be summed. The block convolutions have kernel size $3,$ stride $1$.  
The torso is followed by an LSTM (size specified in \cref{table:hyperparameters2}). In Atari, we do not use an LSTM, but use frame stacking (4) instead.

\textbf{Hyperparameters.}
\cref{table:hyperparameters} lists all the hyperparameters used by STAC and STACX. Most of the hyperparameters are shared across all domains (listed by Value), and follow the reported parameters from the IMPALA paper. Whenever they are not, we list the specific values that are used in each domain (listed by domain). 
\begin{table}[h]
\caption{Hyperparameters table}
\begin{center}
\begin{tabular}{|l|l|l|l|}
    \hline
    Parameter & Value & Atari & Control  \\
    \hline 
    total envirnoment steps & - & 200e6 & 30e6 \\
    optimizer & RMSPROP & - & -\\
    start learning rate & - & $6 \cdot 10^{-4}$ & $10^{-3}$ \\
    end learning rate & - & 0 & $10^{-4}$ \\
    decay & 0.99 & - & -\\
    eps & 0.1 & - & -\\
    batch size (m) & - & 32 & 24\\
    trajectory length (n) & - & 20 & 40\\
    overlap length & - & 0 & 30\\
    $\gamma^{\text{outer}}$ & -& $0.995$ & 0.99  \\
    $\lambda^{\text{outer}}$ & 1 & -& - \\ 
    $\alpha^{\text{outer}}$ & 1 & -& - \\
    $g_e^{\text{outer}}$ & 0.01 & -& - \\
    $g_v^{\text{outer}}$ & 0.25 & -& - \\
    $g_p^{\text{outer}}$ & 1 & -& - \\
    $g^{\text{kl}}$ & 1 & -& - \\
    meta optimizer & Adam & - & -\\
    meta learning rate & $10^{-3}$ & - & - \\
    b1 & 0.9 & - & - \\
    b2 & 0.999 & - & -\\
    eps & 1e-4 & - & -\\
    $\eta^{\text{init}}$ & 4.6 & -& - \\
    \hline
\end{tabular}
\end{center}
\label{table:hyperparameters}
\end{table}

As a design choice, we did not tune many of the hyperparameters. Instead, we use the hyperparameters that were reported for IMPALA in earlier work. For new hyperparameters, we preferred using default values. 

For example, we chose Adam as a meta optimizer with default hyperparameters that were not tuned. For control, we use $\gamma=0.99$, which is the default value of many agents in this domain. The network architectures are quite standard as well and was used in the IMPALA paper. 

We now list a few things that we did try to tune.

\textbf{Number of auxiliary tasks.} We have also experimented with other amounts of auxiliary losses, e.g., having $2, 5$, or $8$ auxiliary loss functions. In Atari, these variations performed better than having a single loss function (STAC) but slightly worse than having $3$. This can be further explained by \cref{fig:adaptivity_jamesbond}, which shows that the auxiliary heads are self-tuned to similar metaparameters. 

\textbf{Behaviour policy.}
We considered two variations of STACX that allow the other heads to act. 

\begin{enumerate}
    \item Random ensemble: The policy head is chosen at random from $[1,..,n],$, and the hyperparameters are differentiated w.r.t the performance of each of the heads in the outer loop.
    \item Average ensemble: The actor policy is defined to be the average logits of the heads, and we learn one additional head for the value function of this policy.
\end{enumerate} 

The metagradient in the outer loop is taken w.r.t the actor policy, and /or, each one of the heads individually. While these extensions seem interesting, in all of our experiments, they always led to a small decrease in performance when compared to our auxiliary task agent without these extensions. Similar findings were reported in \citep{fedus2019hyperbolic}.

\textbf{KL coefficient}. When we introduced this loss function, we did not use a coefficient for it. Thus, it had the default value 1. Later on, we revisited this design choice and tested what happens when we use $g^{\text{KL}} \in \{ 0, 0.3, 1, 2\}.$ We observed that the value of this parameter could affect our results, but not significantly, and we chose to remain with the default value (1), which seemed to perform the best. 

\subsection{Resource Usage}

\begin{table}[h]
\caption{Run times in minutes}
\begin{center}
    \begin{tabular}{|l|l|l|l|l|}
        \hline
        Domain & IMPALA & STAC & IMPALA Aux & STACX  \\
        \hline 
        RWRL & 44 & 44.3 & 55.3 & 56 \\ \hline 
        DM control, features & 30.1 & 30.3  & 36.8 & 36.9 \\\hline
        DM control, pixels & 93 & 93  & 97 & 97 \\\hline 
        Atari & 70 & 71  & 83 & 84  \\ \hline
        % Atari & 105 & 149  & 106 & 153  \\ \hline
    \end{tabular}
\end{center}
\label{table:runtime}
\end{table}

The average run time in the different environments is reported in \cref{table:runtime}. Thanks to massively parallelism, with modern hardware such as TPUs, agents are often bottlenecked by data in-feed rather than actual computation. This is true despite the use of fairly deep networks. As a result, we can increase the exact amount of computation performed on each step with a modest impact on the runtime. Inspecting \cref{table:runtime}, we can see that self-tuning indeed results with a very mild increase in run time. However, this does not mean that self-tuning costs the same amount of compute, which is hard to measure. 

To further investigate this, we measured the run time in Atari with a second hardware configuration (the combination of distributed CPUs with a GPU learner from section \ref{compute}). The run times of the different agents in Atari were 105, 129, 106, and 133 (for IMPALA, STAC, IMPALA Aux and STACX respectively).  With this hardware, the run time of all the agents was longer. When we compare the run time of the different agents, we can see that self-tuning required about $25\%$ more time, while the extra run time from having auxiliary loss functions was negligible. 

To conclude, the run times of the different agents is hardware specific. We observed that overall STAC and STACX result in slightly more compute than the baseline agent. 

\subsection{Reproducing Baseline Algorithms}
\label{sec:rep_xu}
Inspecting the results in \cref{fig:generality}, one may notice small differences between the results of IMPALA and using meta gradients to tune only $\lambda, \gamma$  compared to the results that were reported in \citep{xu2018meta}. 

We investigated the possible reasons for these differences. First, our method was implemented in a different codebase. Our code is written in JAX, compared to the implementation in \citep{xu2018meta} that was written in TensorFlow. This may explain the small difference in final performance between our IMPALA baseline that achieved $243\%$ median score (\cref{fig:generality}) and the result of Xu et al., which is slightly higher ($257.1\%$). 

Second, Xu et al. observed that embedding the $\gamma$ hyperparameter into the $\theta$ network improved their results significantly, reaching a final performance (when learning $\gamma,\lambda$) of $287.7\%$ when self-tuning only $\gamma$ and $\lambda$ (see section 1.4 in \citep{xu2018meta} for more details). When tuning only $\gamma,$ \citet{xu2018meta} report that without using the embedding, the performance of their meta gradient agent drops to $183\%$, even below the IMPALA baseline, where when using the $\gamma$ embedding, the performance increases to $267.9\%$ (for self-tuning only $\gamma$). 

When self-tuning only $\gamma$ but without using embedding, we also observed a slight decrease in performance ($243\%$->$240\%$), but not as significant as in \citep{xu2018meta}. We further investigated this difference by introducing the $\gamma$ embedding into our architecture. With $\gamma$ embedding, our method achieved a score of $280.6\%$ (for self tuning only $\lambda, \gamma$), which almost reproduces the results in \citep{xu2018meta}. 

We also introduced the same embedding mechanism to STACX when self-tuning all the metaparameters. In this case, for auxiliary loss $i$ we embed $\gamma_i$. We experimented with two variants, one that shares the embedding weights across the auxiliary tasks and one that learns a specific embedding for each auxiliary task. Both of these variants performed similarly ($306.8\%$, $307.7\%$ respectively), which is better than the result for self-tuning only $\gamma, \lambda$ ($280.6\%$). However, STACX performed better without the embedding ($364\%$), so we did not use the $\gamma$ embedding in our architecture. We leave it to future work to further investigate methods of combining the embedding mechanisms with the auxiliary loss functions. 

\clearpage

\section{Analysis of Leaky V-trace}
\label{sec:proof}
Define the Leaky V-trace operator $\tilde{\mathcal{R}}$:
\begin{equation}
    \label{eq:leaky_Bellman}
    \tilde{\mathcal{R}}V(x) = V(x) + \mathbb{E}_\mu \left[ 
    \sum_{t\ge0} \gamma ^t \left(\Pi_{i=0}^{t-1}\tilde c_i\right) \tilde \rho_t \left(r_t + \gamma V(x_{t+1}) - V(x_t)\right)|x_0=x,\mu)
\right],
\end{equation}

where the expectation $\mathbb{E}_\mu$ is with respect to the behaviour policy $\mu$ which has generated the trajectory $(x_t)_{t\ge0}$, i.e., $x_0 = x, x_{t+1} \sim P(\cdot|x_t, a_t), a_t \sim \mu(\cdot|x_t)$. Similar to \citep{espeholt2018impala}, we consider the infinite-horizon operator but very similar results hold for the n-step truncated operator.

Let 
$$\text{IS}(x_t) = \frac{\pi(a_t|x_t)}{\mu(a_t|x_t)},$$
be importance sampling weights, let 
$$\rho_t = \min(\bar \rho,\text{IS}(x_t)), \enspace c_t =\min(\bar c,\text{IS}(x_t)),$$
be truncated importance sampling weights with $\bar \rho \ge \bar c$, and let 
$$\tilde \rho_t =\alpha_\rho \rho_t + (1-\alpha_\rho)\text{IS}(x_t), \enspace \tilde c_t = \alpha_c c_t + (1-\alpha_c)\text{IS}(x_t) $$
be the Leaky importance sampling weights with leaky coefficients $\alpha_\rho \ge \alpha_c.$

\begin{theorem}[Restatement of \cref{thm:leaky}]
     Assume that there exists $\beta \in (0, 1]$ such that $\mathbb{E}_\mu \rho_0 \ge \beta$. Then the operator $\tilde{\mathcal{R}}$ defined by \cref{eq:leaky_Bellman} has a unique fixed point $\tilde V ^{\tilde \pi}$, which is the value function of the policy $ \pi_{\bar \rho, \alpha_\rho}$ defined by
$$
     \pi_{\bar \rho, \alpha_\rho} = \frac{\alpha_\rho\min \left(\bar \rho \mu(a|x), \pi(a|x)\right) +(1-\alpha_\rho)\pi(a|x)}{\sum_b \alpha_\rho \min \left(\bar \rho \mu(b|x), \pi(b|x)\right)+(1-\alpha_\rho)\pi(b|x)},
$$
    Furthermore, $\tilde{\mathcal{R}}$ is a $\tilde \eta$-contraction mapping in sup-norm, with $$
    \tilde{\eta} = \gamma ^{-1} - (\gamma^{-1}-1)\mathbb{E}_\mu \left[ \sum_{t\ge 0} \gamma^t \left( \Pi _{i=0} ^{t-2}\tilde c_i \right) \tilde \rho_{t-1} \right] \le 1-(1-\gamma)(\alpha_\rho \beta + 1 - \alpha_\rho)
    <1, $$
    where $\tilde c_{-1} =1, \tilde \rho_{-1}=1$ and $\Pi_{s=0}^{t-2}c_s = 1$ for $t=0,1.$
\end{theorem}
\begin{proof}
The proof follows the proof of V-trace from \citep{espeholt2018impala} with adaptations for the leaky V-trace coefficients. We have that
$$
\tilde{\mathcal{R}}V_1(x) -  \tilde{\mathcal{R}}V_2(x) = \mathbb{E}_\mu \sum_{t\ge0} \gamma^t \left(\Pi_{s=0}^{t-2}c_s \right) \left[\tilde \rho_{t-1} - \tilde c_{t-1}\tilde \rho_t \right] \left(V_1(x_t)-V_2(x_t)\right).
$$

Denote by $\tilde \kappa _t = \rho_{t-1} - \tilde c_{t-1}\tilde \rho_t,$ and notice that
$$\mathbb{E}_{\mu} \tilde \rho_t = \alpha_\rho \mathbb{E}_{\mu} \rho_t + (1 - \alpha_\rho) \mathbb{E}_{\mu} \text{IS}(x_t)  \le 1,$$
since $\mathbb{E}_\mu\text{IS}(x_t)=1,$ and therefore, $\mathbb{E}_\mu\rho_t\le1.$ Furthermore, since $\bar \rho\ge \bar c$ and $\alpha_\rho \ge \alpha_c,$ we have that $\forall t, \tilde \rho _{t} \ge \tilde c_t.$ Thus, the coefficients $\tilde \kappa _t $ are non negative in expectation, since

$$
\mathbb{E}_{\mu} \tilde \kappa_t = \mathbb{E}_{\mu}\tilde \rho_{t-1} - \tilde c_{t-1}\tilde \rho_t \ge \mathbb{E}_{\mu}\tilde c_{t-1} (1-\tilde \rho _t) \ge 0.
$$
Thus, $V_1(x)- V_2(x)$ is a linear combination of the values $V_1 - V_2$ at the other states, weighted by non-negative coefficients whose sum is 

\begin{align}
     &\sum_{t\ge 0} \gamma^t \mathbb{E}_\mu \left( \Pi _{s=0} ^{t-2}\tilde c_s \right) \left[ \tilde \rho_{t-1} -\tilde c_{t-1} \tilde \rho_t\right]\nonumber \\
     = & \sum_{t\ge 0} \gamma^t \mathbb{E}_\mu \left( \Pi _{s=0} ^{t-2}\tilde c_s \right) \tilde \rho_{t-1} -  \sum_{t\ge 0} \gamma^t \mathbb{E}_\mu \left( \Pi _{s=0} ^{t-1}\tilde c_s \right)\tilde  \rho_t \nonumber \\
      = & \gamma^{-1} - (\gamma^{-1}-1)\sum_{t\ge 0} \gamma^t \mathbb{E}_\mu \left( \Pi _{s=0} ^{t-2}\tilde c_s \right) \tilde \rho_{t-1} \nonumber \\
      \le & \gamma^{-1} - (\gamma^{-1}-1)(1+\gamma \mathbb{E}_\mu \tilde \rho_0) \label{line:sum}\\
      = & 1 - (1-\gamma)\mathbb{E}_\mu \tilde \rho_0 \nonumber\\
      = & 1 - (1-\gamma)\mathbb{E}_\mu \left( \alpha_\rho \rho_0 + (1-\alpha_\rho)\text{IS}(x_0)\right) \nonumber\\
     \le & 1 - (1-\gamma) \left (\alpha_\rho \beta + 1-\alpha_\rho\right)<1 ,\label{eq:beta1}
 \end{align}
where \cref{line:sum} holds since we expanded only the first two elements in the sum, and all the elements in this sum are positive, and \cref{eq:beta1} holds by the assumption.

We deduce that $\|\tilde{\mathcal{R}}V_1(x) - \tilde{\mathcal{R}}V_2(x)\|_\infty \le \tilde \eta \| V_1(x)-V_2(x)\|_\infty,$  with $\tilde \eta = 1 - (1-\gamma) \left (\alpha_\rho \beta + 1-\alpha_\rho\right) <1,$ so $\tilde{\mathcal{R}}$ is a contraction mapping. Furthermore, we can see that the parameter $\alpha_\rho$ controls the contraction rate, for $\alpha_\rho=1 $ we get the contraction rate of V-trace $\tilde \eta = 1-(1-\gamma)\beta$ and as $\alpha_\rho$ gets smaller with get better contraction as with $\alpha_\rho=0$ we get that $\tilde \eta = \gamma.$

Thus $\tilde{\mathcal{R}}$ possesses a unique fixed point. Let us now prove that this fixed point is $V^{\pi_{\bar \rho, \alpha_\rho}},$ where 
\begin{equation}
     \pi_{\bar \rho, \alpha_\rho} = \frac{\alpha_\rho\min \left(\bar \rho \mu(a|x), \pi(a|x)\right) +(1-\alpha_\rho)\pi(a|x)}{\sum_b \alpha_\rho \min \left(\bar \rho \mu(b|x), \pi(b|x)\right)+(1-\alpha_\rho)\pi(b|x)},
\end{equation}  
is a policy that mixes the target policy with the V-trace policy. 

We have: 
\begin{align*}
& \mathbb{E}_\mu \left[ \tilde \rho_t (r_t + \gamma V^{\pi_{\bar \rho, \alpha_\rho}}(x_{t+1}) - V^{\pi_{\bar \rho, \alpha_\rho}}(x_t))|x_t\right] \\
= &  \mathbb{E}_\mu \left(\alpha_\rho \rho_t + (1-\alpha_\rho)\text{IS}(x_t)\right)(r_t + \gamma V^{\pi_{\bar \rho, \alpha_\rho}}(x_{t+1}) - V^{\pi_{\bar \rho, \alpha_\rho}}(x_t))  \\
= & \sum_a \mu(a|x_t) \left(\alpha_\rho \min\left(\bar \rho,\frac{\pi(a|x_t)}{\mu(a|x_t)}\right) + (1-\alpha_\rho)\frac{\pi(a|x_t)}{\mu(a|x_t)}\right)(r_t + \gamma V^{\pi_{\bar \rho, \alpha_\rho}}(x_{t+1}) - V^{\pi_{\bar \rho, \alpha_\rho}}(x_t))\\
= & \sum_a \left(\alpha_\rho \min\left(\bar \rho \mu(a|x_t),\pi(a|x_t)\right) + (1-\alpha_\rho)\pi(a|x_t)\right)(r_t + \gamma V^{\pi_{\bar \rho, \alpha_\rho}}(x_{t+1}) - V^{\pi_{\bar \rho, \alpha_\rho}}(x_t))\\
= & \sum_a \pi_{\bar \rho, \alpha_\rho}(a|x_t)(r_t + \gamma V^{\pi_{\bar \rho, \alpha_\rho}}(x_{t+1}) - V^{\pi_{\bar \rho, \alpha_\rho}}(x_t)) \cdot \sum_b \left(\alpha_\rho \min\left(\bar \rho \mu(b|x_t),\pi(b|x_t)\right) + (1-\alpha_\rho)\pi(b|x_t)\right) =0,
\end{align*}
where we get that the left side (up to the summation on $b$) of the last equality equals zero since this is the Bellman equation for $V^{\pi_{\bar \rho, \alpha_\rho}}.$ We deduce that $\tilde{\mathcal{R}}V^{\pi_{\bar \rho, \alpha_\rho}} = V^{\pi_{\bar \rho, \alpha_\rho}},$ thus, $V^{\pi_{\bar \rho, \alpha_\rho}}$ is the unique fixed point of $\tilde{\mathcal{R}}.$
\end{proof}

\clearpage

\section{Individual domain learning curves in DM control and RWRL} 
\label{sec:dm_control}
\begin{figure}[h]
    \centering
    \includegraphics[width=\linewidth]{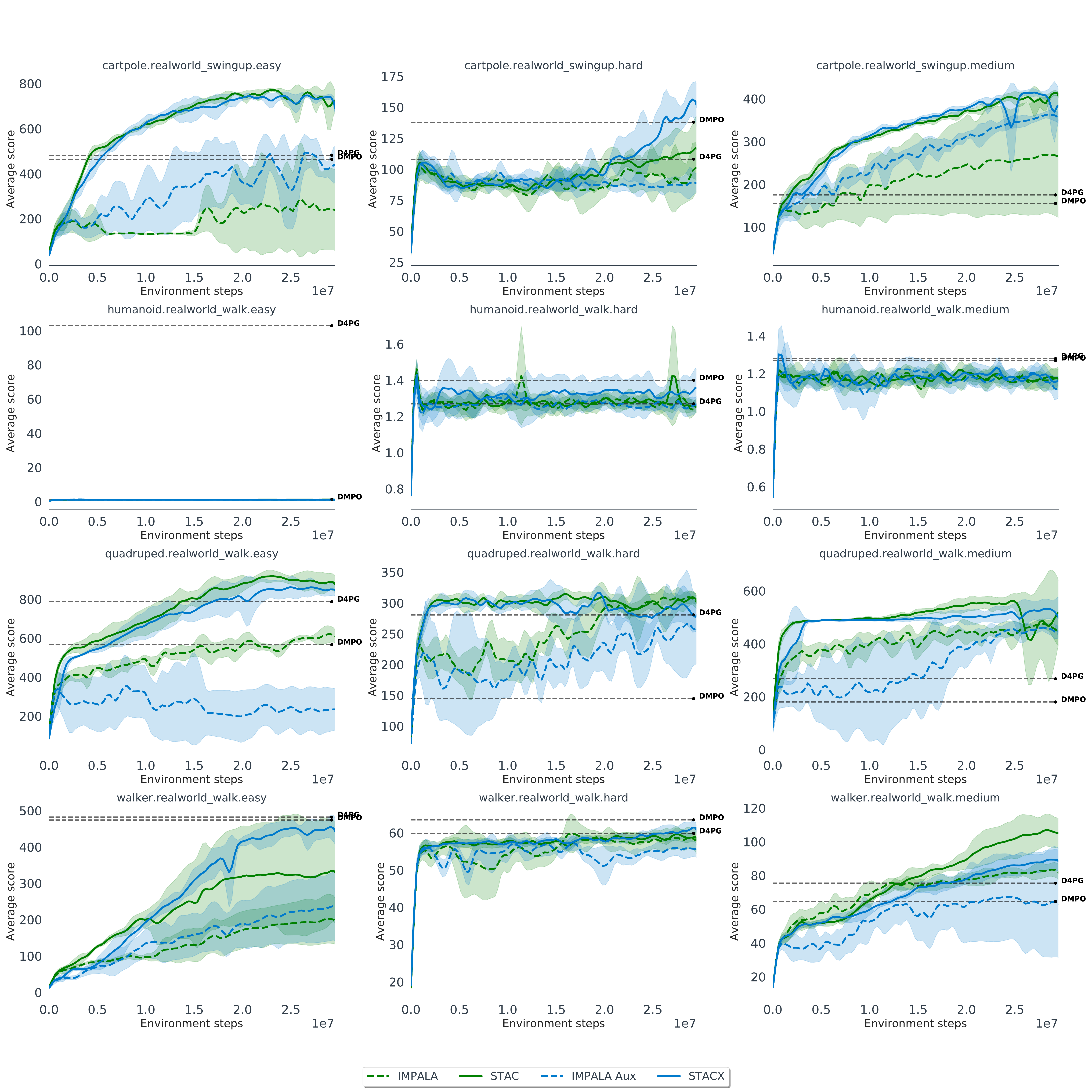}
\caption{Learning curves in specific domains of the RWRL challenge \citep{dulac2020empirical}. In each domain we report the mean, averaged over 3 seeds, of each method with std confidence intervals as shaded areas. In addition, we report the scores for two baselines (D4PG, DMPO) at the end of training, taken from \citep{dulac2020empirical}. }
\label{fig:rwrl_lvls}
\end{figure}

\begin{figure}[h]
    \centering
    \includegraphics[width=\linewidth]{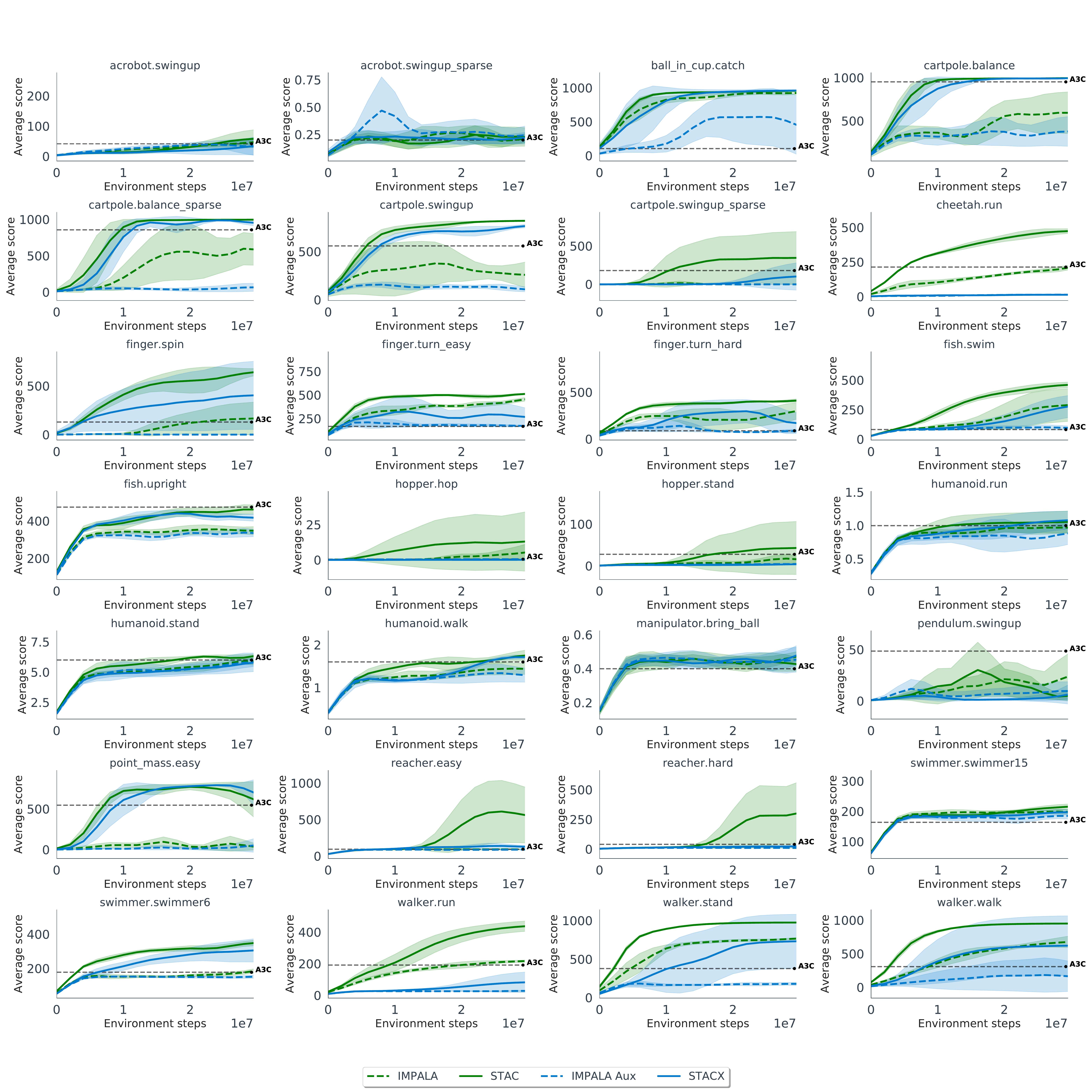}
\caption{Learning curves in specific domains of the DM control suite \citep{tassa2018deepmind} using the features. In each domain we report the mean, averaged over 3 seeds, of each method with std confidence intervals as shaded areas. In addition, we report the scores of the A3C baseline at the end of training, taken from \citep{tassa2018deepmind}. }
\label{fig:low_lvls}
\end{figure}

\begin{figure}[h]
    \centering
    \includegraphics[width=\linewidth]{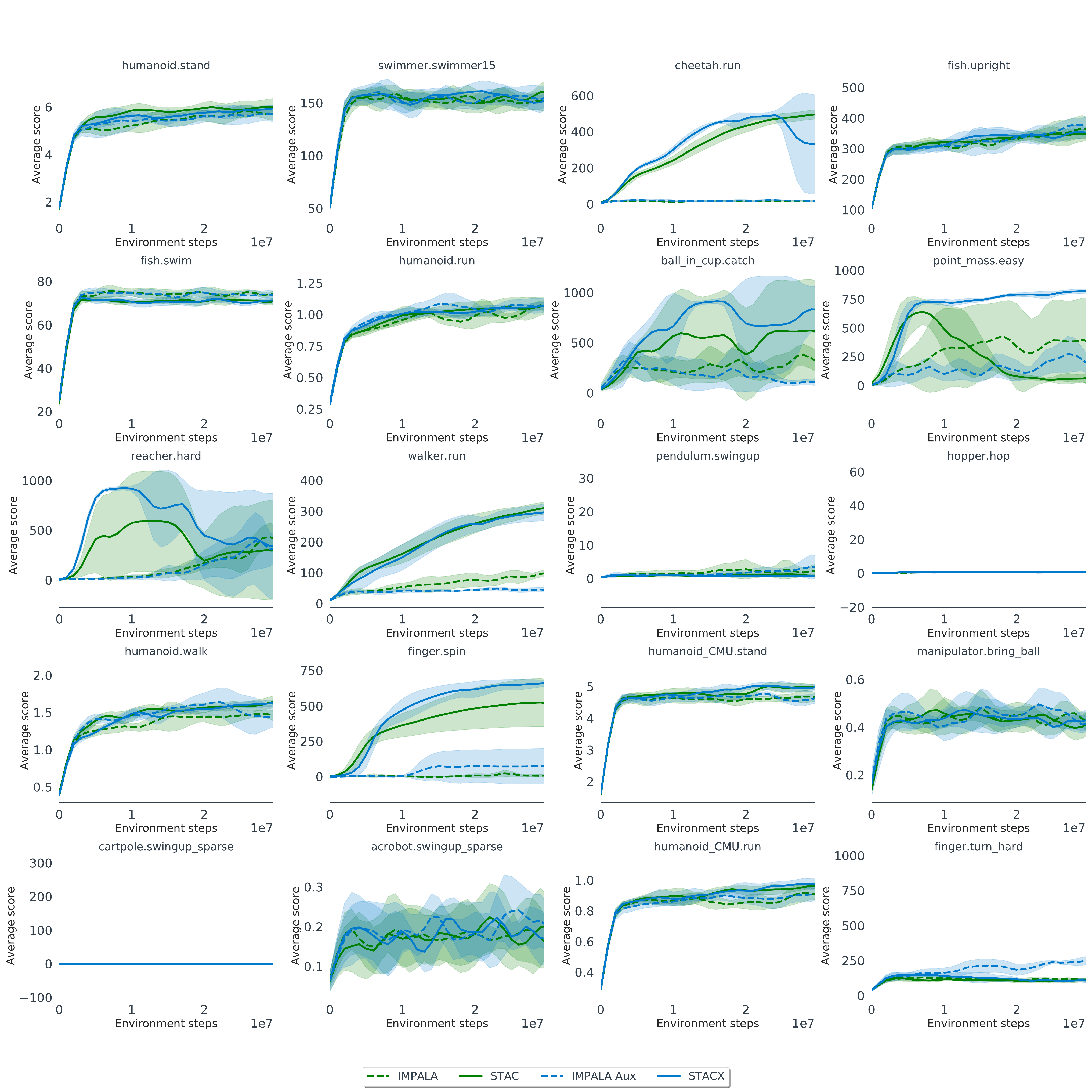}
\caption{Learning curves in specific domains of the DM control suite \citep{tassa2018deepmind} using pixels. In each domain we report the mean, averaged over 3 seeds, of each method with std confidence intervals as shaded areas.}
\label{fig:pixels_lvls}
\end{figure}

\clearpage

\section{Relative improvement in percents of STACX and STAC over IMPALA in Atari}
\label{sec:rel_atari}
\begin{figure}[h]
    \centering
    \includegraphics[width=0.6\linewidth]{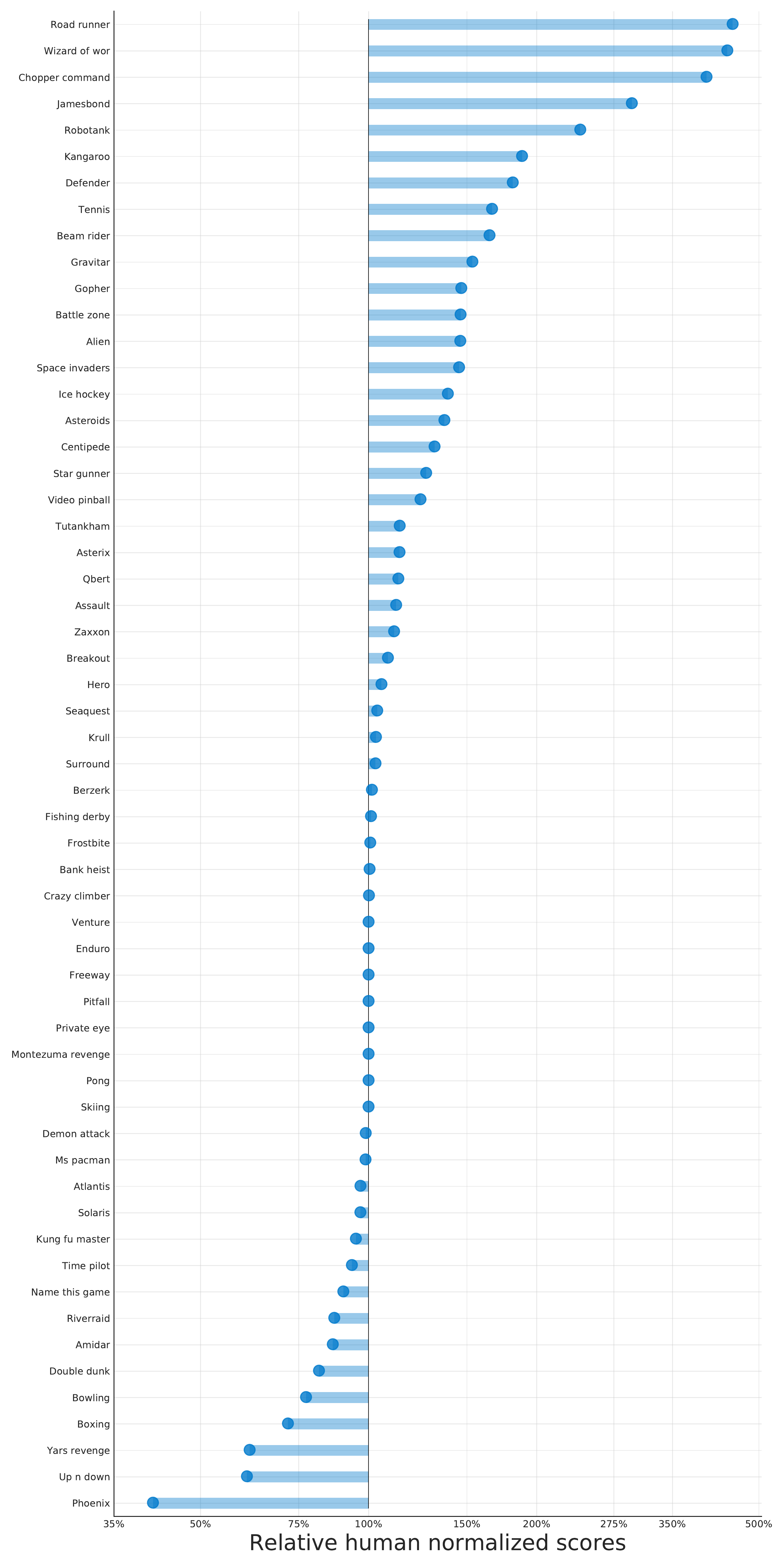}
\caption{Mean human-normalized scores after 200M frames, relative improvement in percents of STACX over IMPALA.}
\label{fig:relative}
\end{figure}

\begin{figure}[h]
    \centering
    \includegraphics[width=0.6\linewidth]{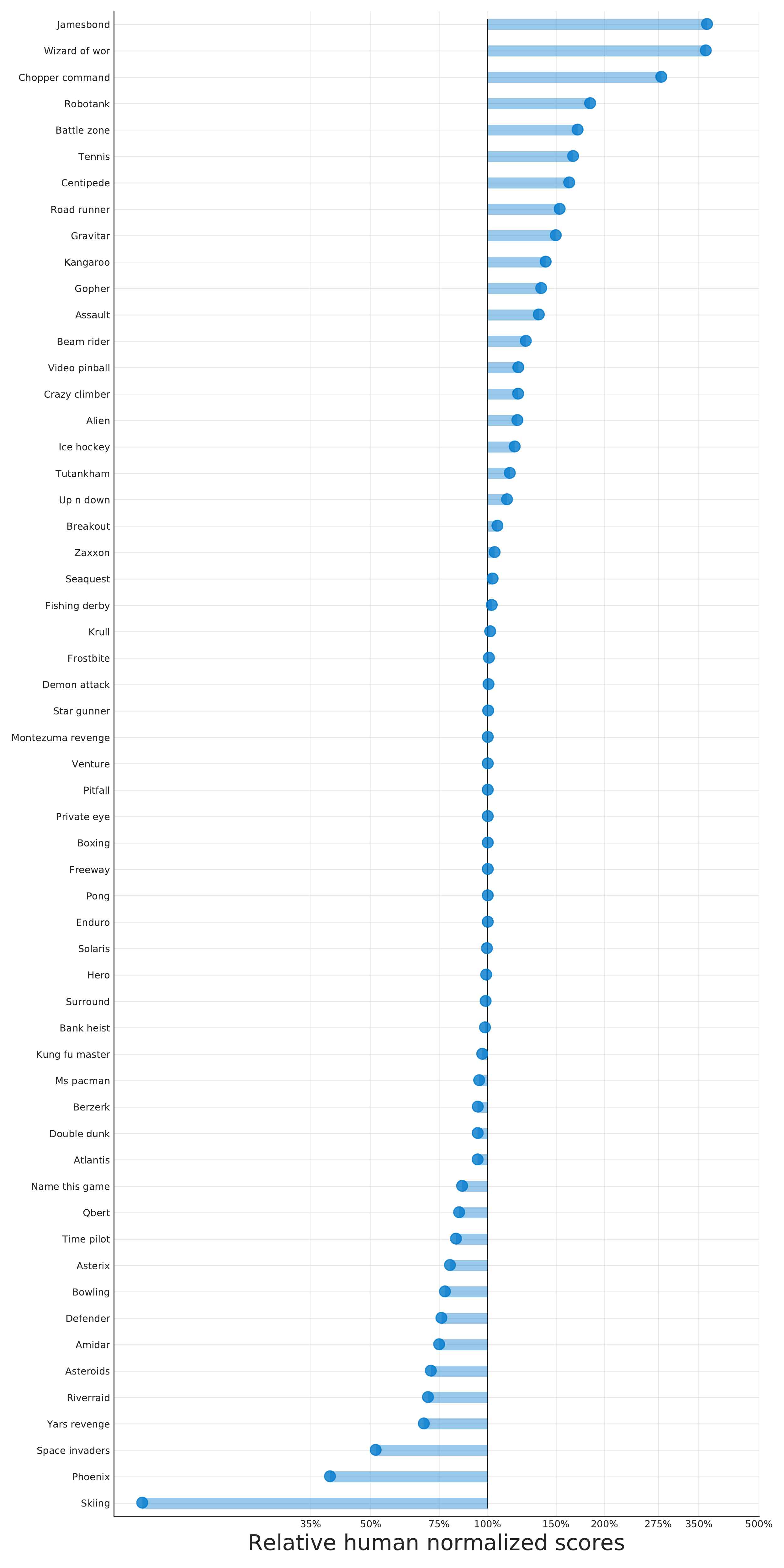}
\caption{Mean human-normalized scores after 200M frames, relative improvement in percents of STAC over the IMPALA baseline.}
\label{fig:relative2}
\end{figure}

\clearpage

\section{Individual game learning curves} 
\label{sec:lc_atari}

For clarity, we repeat a comment from the main paper regarding the values of the metaparameters. We present the values of the metaparameters used in the inner loss, i.e., after we apply a sigmoid activation. But to have a single scale for all the metaparameters ($\eta \in [0,1]$), we present the loss coefficients $g_e,g_v,g_p$ without scaling them by the respective value in the outer loss. For example, the value of the entropy weight $g_e$ is further multiplied by $g_e^{\text{outer}}=0.01$ when used in the inner loss.

\begin{figure}[h]
    \centering
    \includegraphics[width=\linewidth]{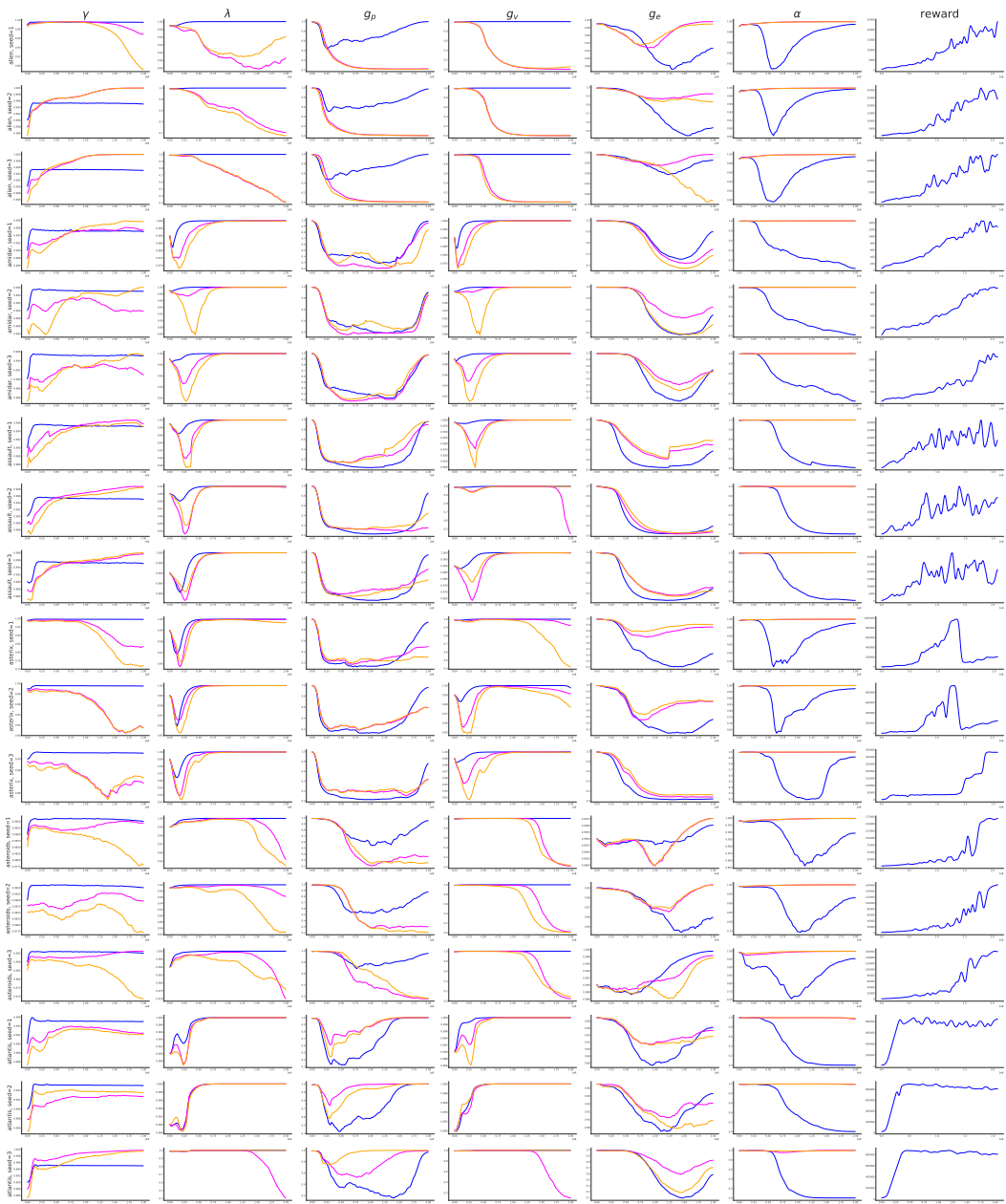}
\caption{Meta parameters and reward in each Atari game (and seed) during learning. Different colors correspond to different heads, blue is the main (policy) head.}
\end{figure}

\begin{figure}[h]
    \centering
    \includegraphics[width=\linewidth]{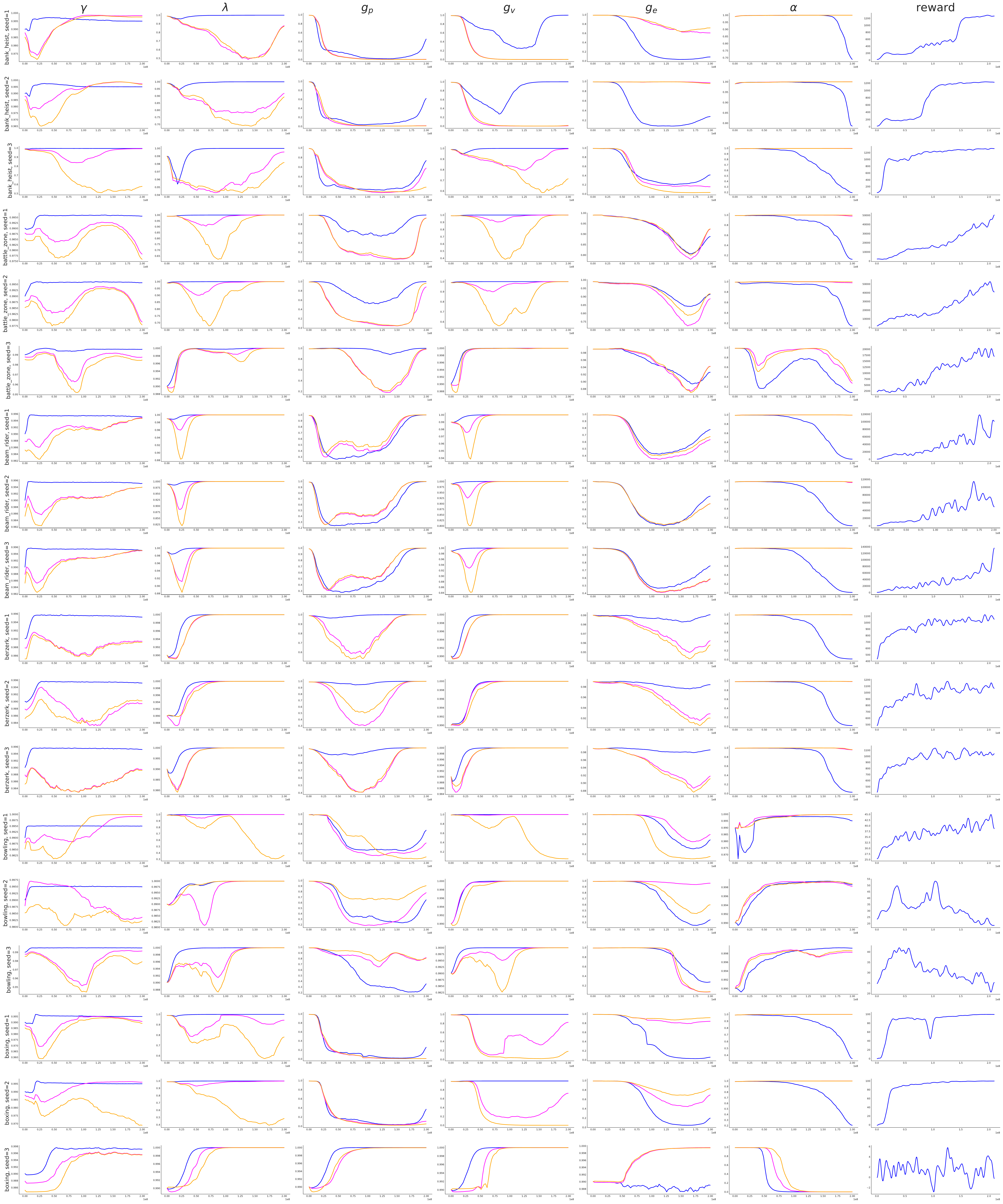}
\caption{Meta parameters and reward in each Atari game (and seed) during learning. Different colors correspond to different heads, blue is the main (policy) head.}
\end{figure}
\begin{figure}[h]
    \centering
    \includegraphics[width=\linewidth]{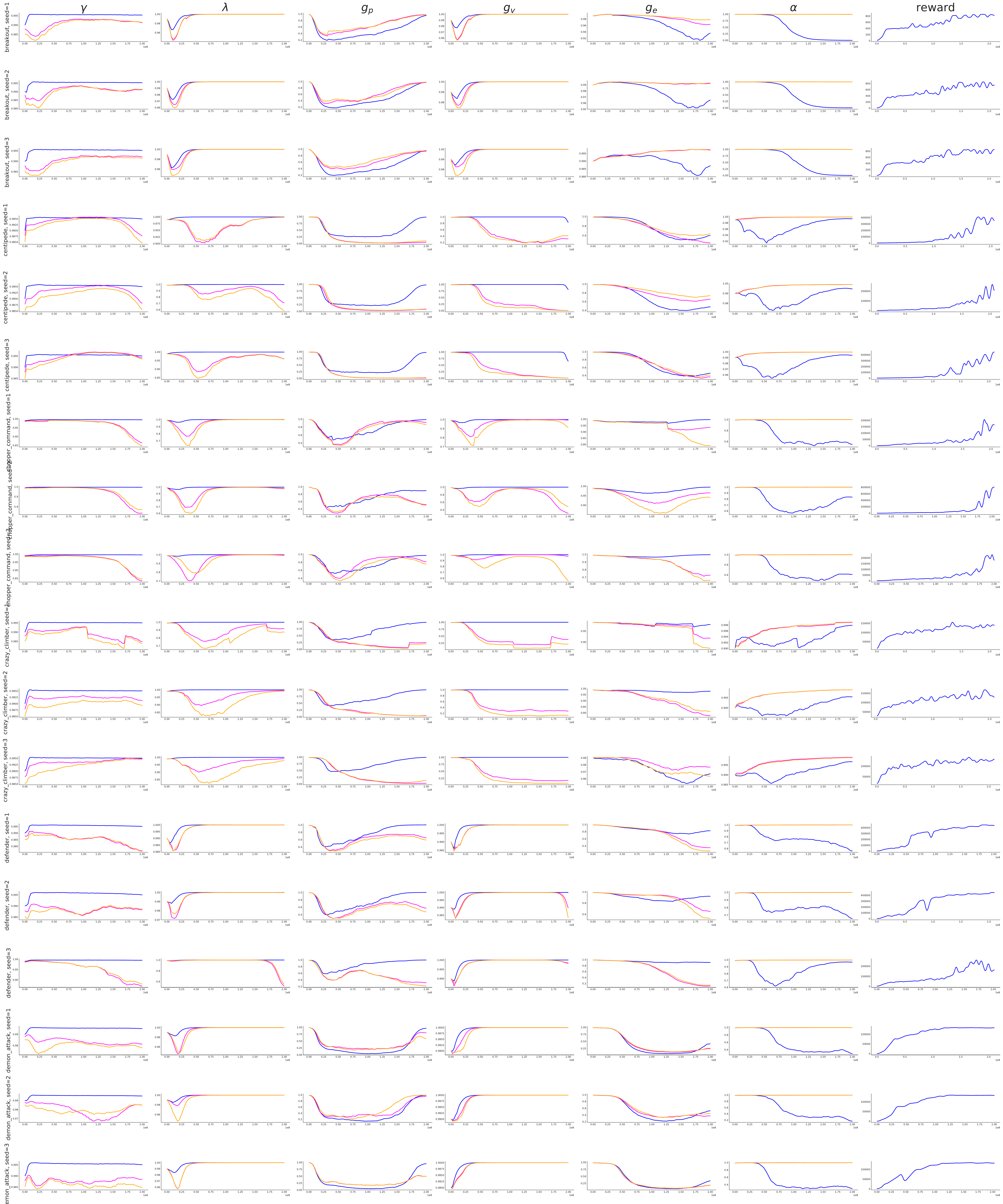}
\caption{Meta parameters and reward in each Atari game (and seed) during learning. Different colors correspond to different heads, blue is the main (policy) head.}
\end{figure}
\begin{figure}[h]
    \centering
    \includegraphics[width=\linewidth]{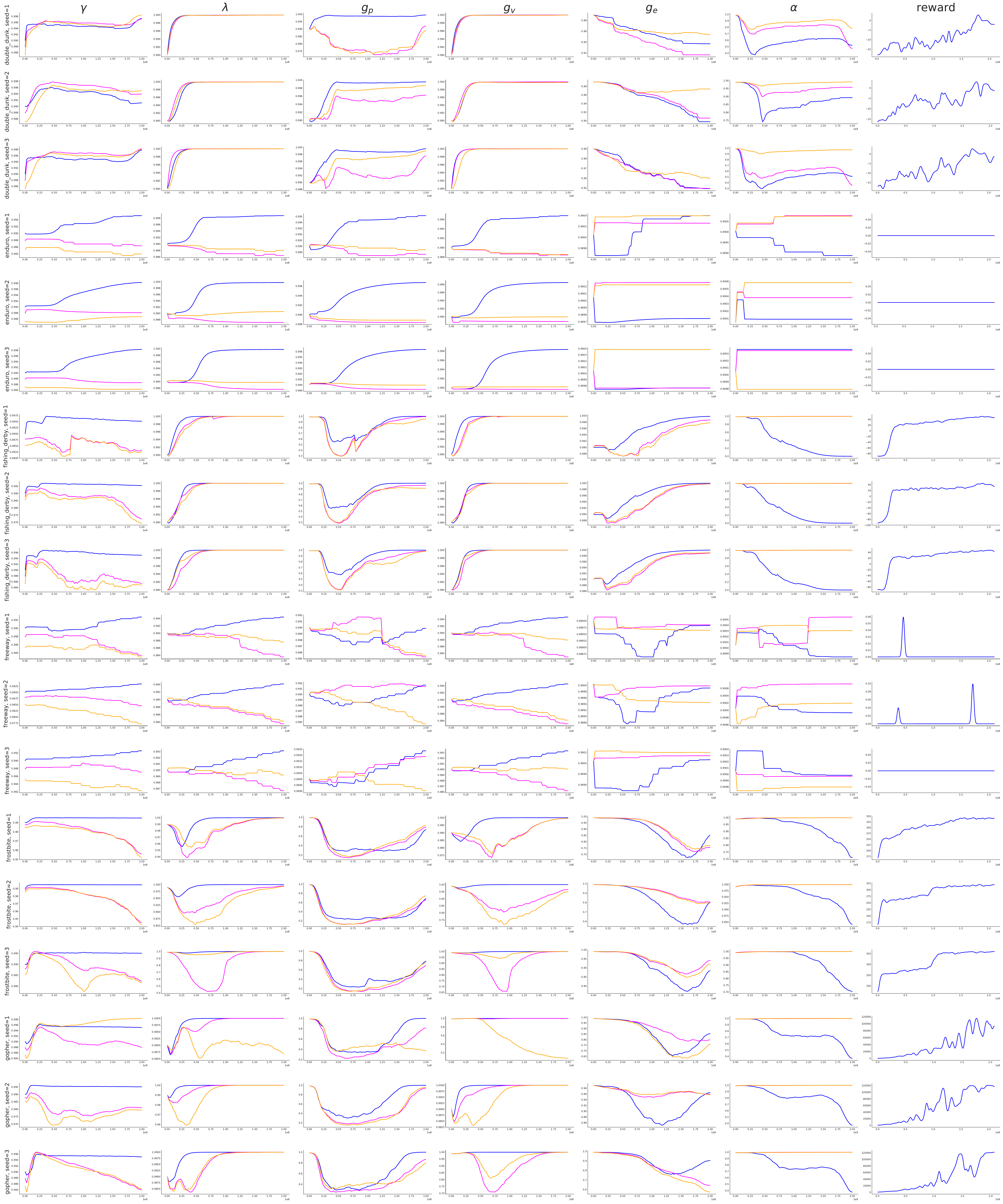}
\caption{Meta parameters and reward in each Atari game (and seed) during learning. Different colors correspond to different heads, blue is the main (policy) head.}
\end{figure}
\begin{figure}[h]
    \centering
    \includegraphics[width=\linewidth]{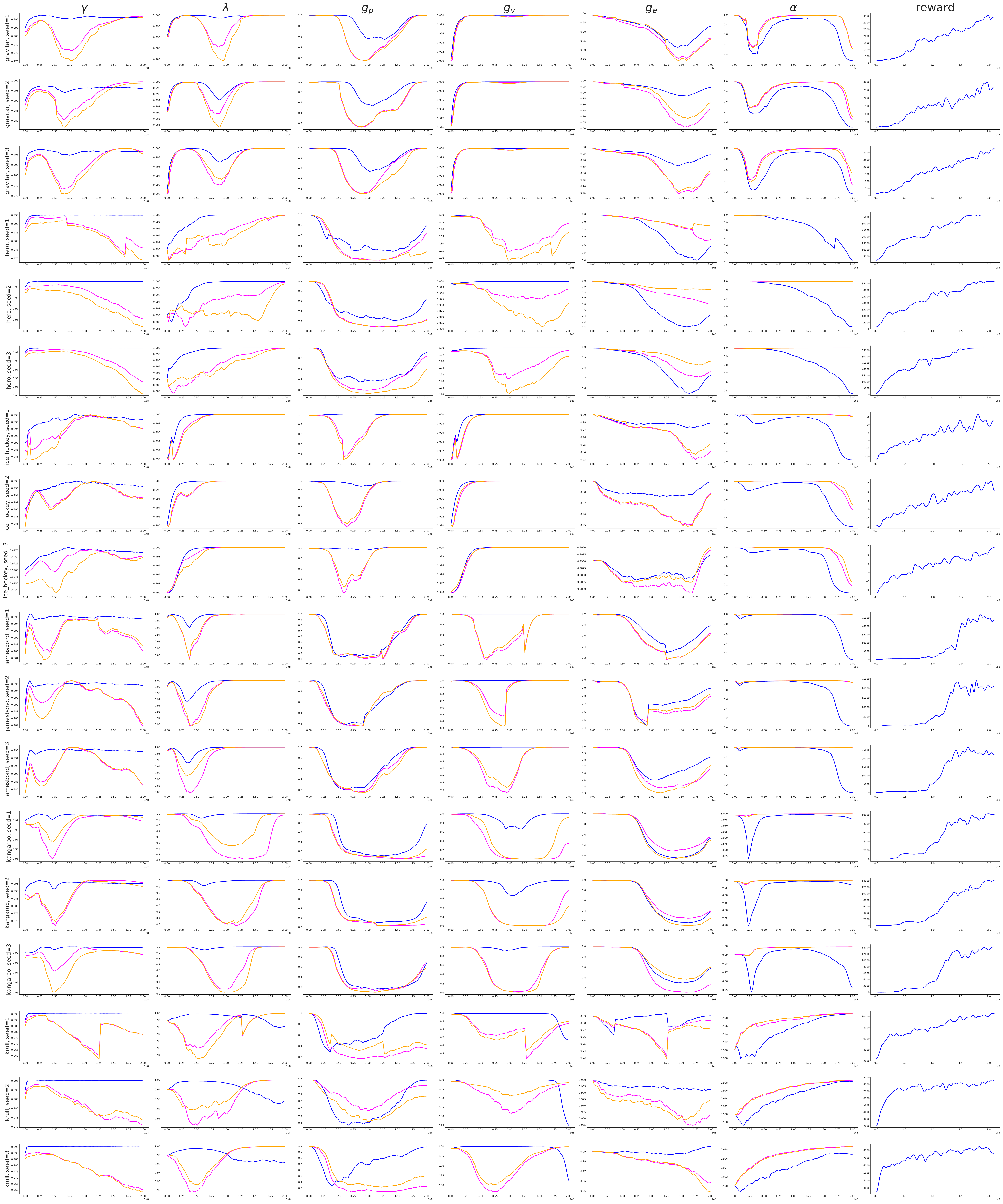}
\caption{Meta parameters and reward in each Atari game (and seed) during learning. Different colors correspond to different heads, blue is the main (policy) head.}
\end{figure}
\begin{figure}[h]
    \centering
    \includegraphics[width=\linewidth]{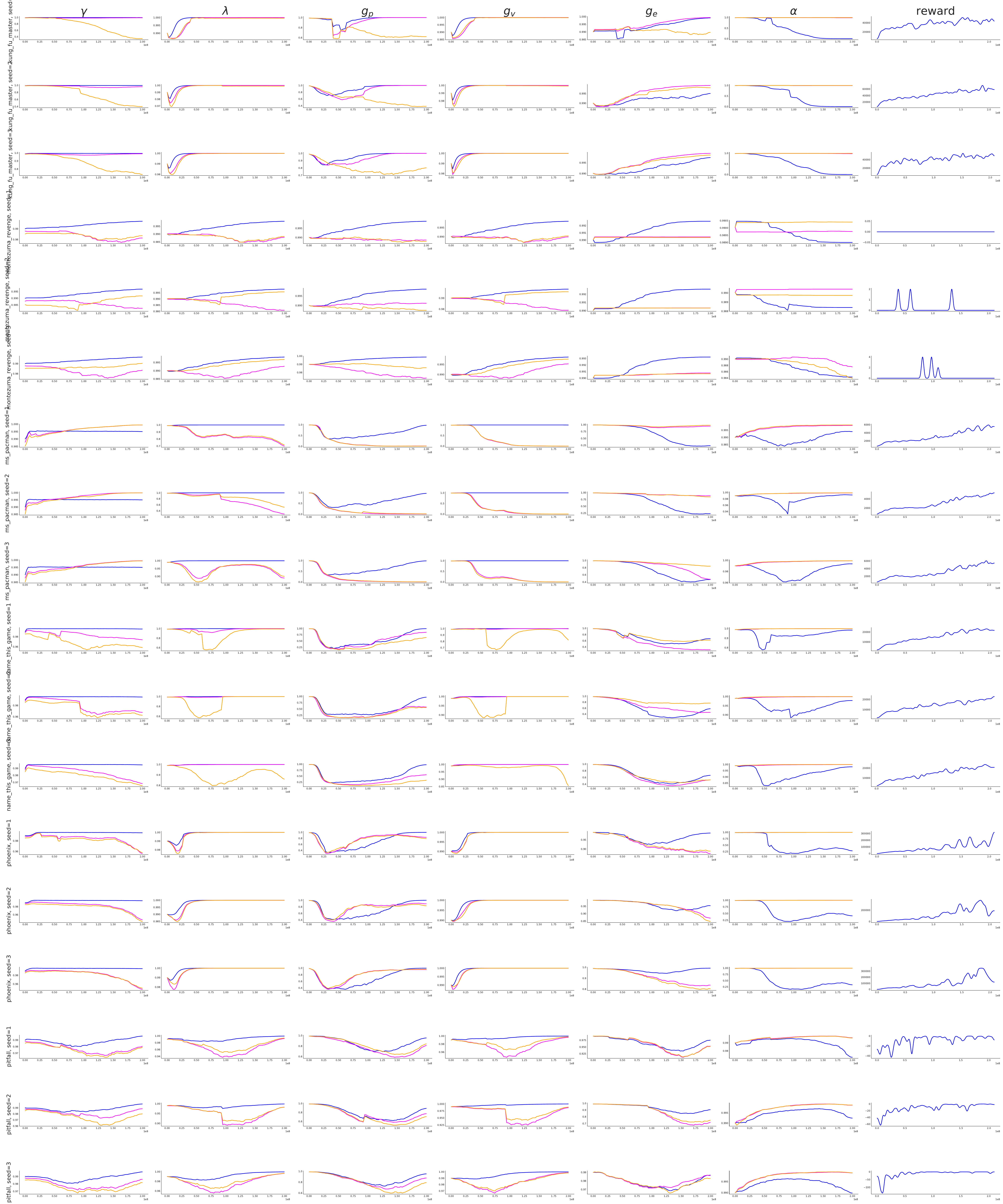}
\caption{Meta parameters and reward in each Atari game (and seed) during learning. Different colors correspond to different heads, blue is the main (policy) head.}
\end{figure}
\begin{figure}[h]
    \centering
    \includegraphics[width=\linewidth]{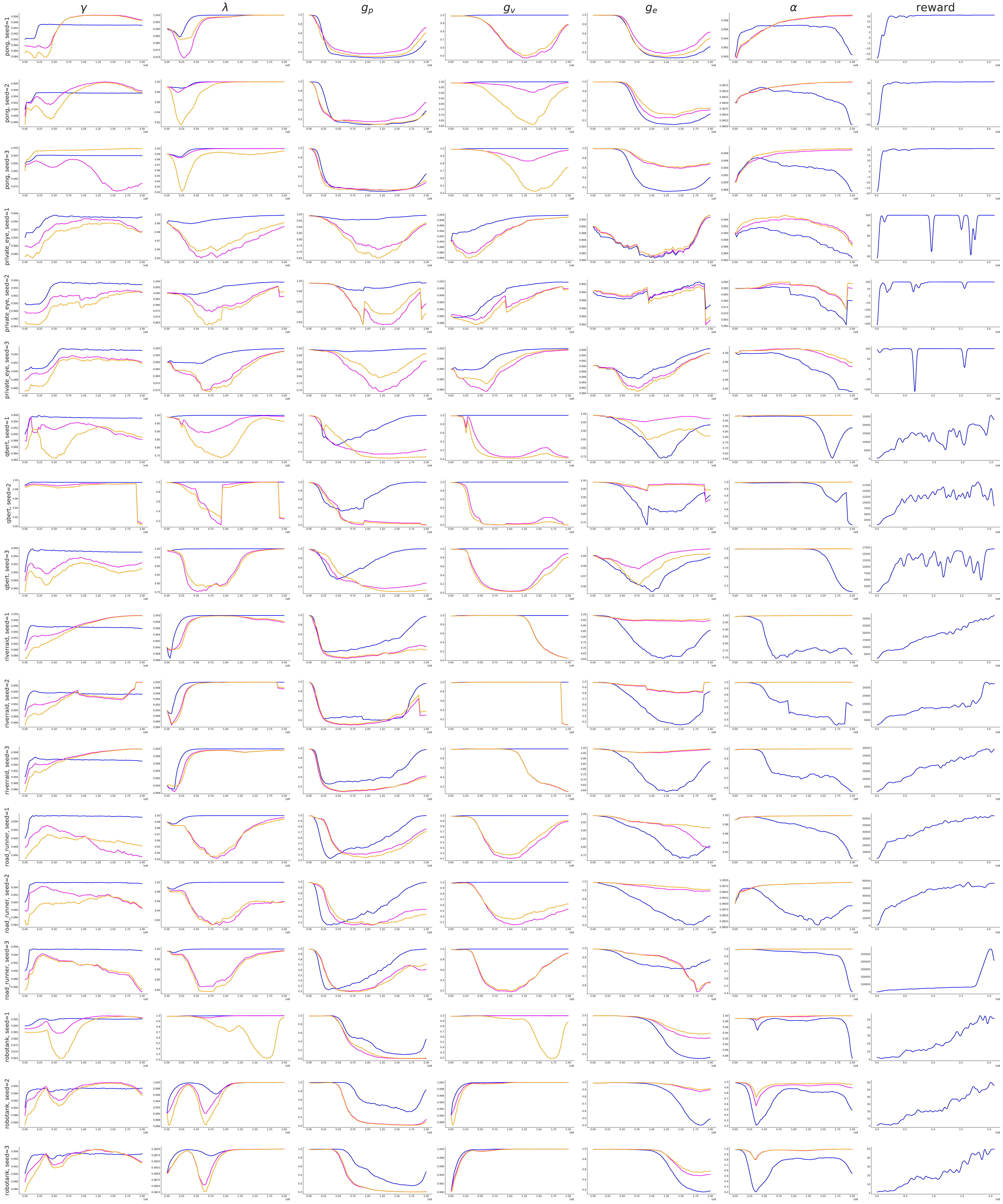}
\caption{Meta parameters and reward in each Atari game (and seed) during learning. Different colors correspond to different heads, blue is the main (policy) head.}
\end{figure}
\begin{figure}[h]
    \centering
    \includegraphics[width=\linewidth]{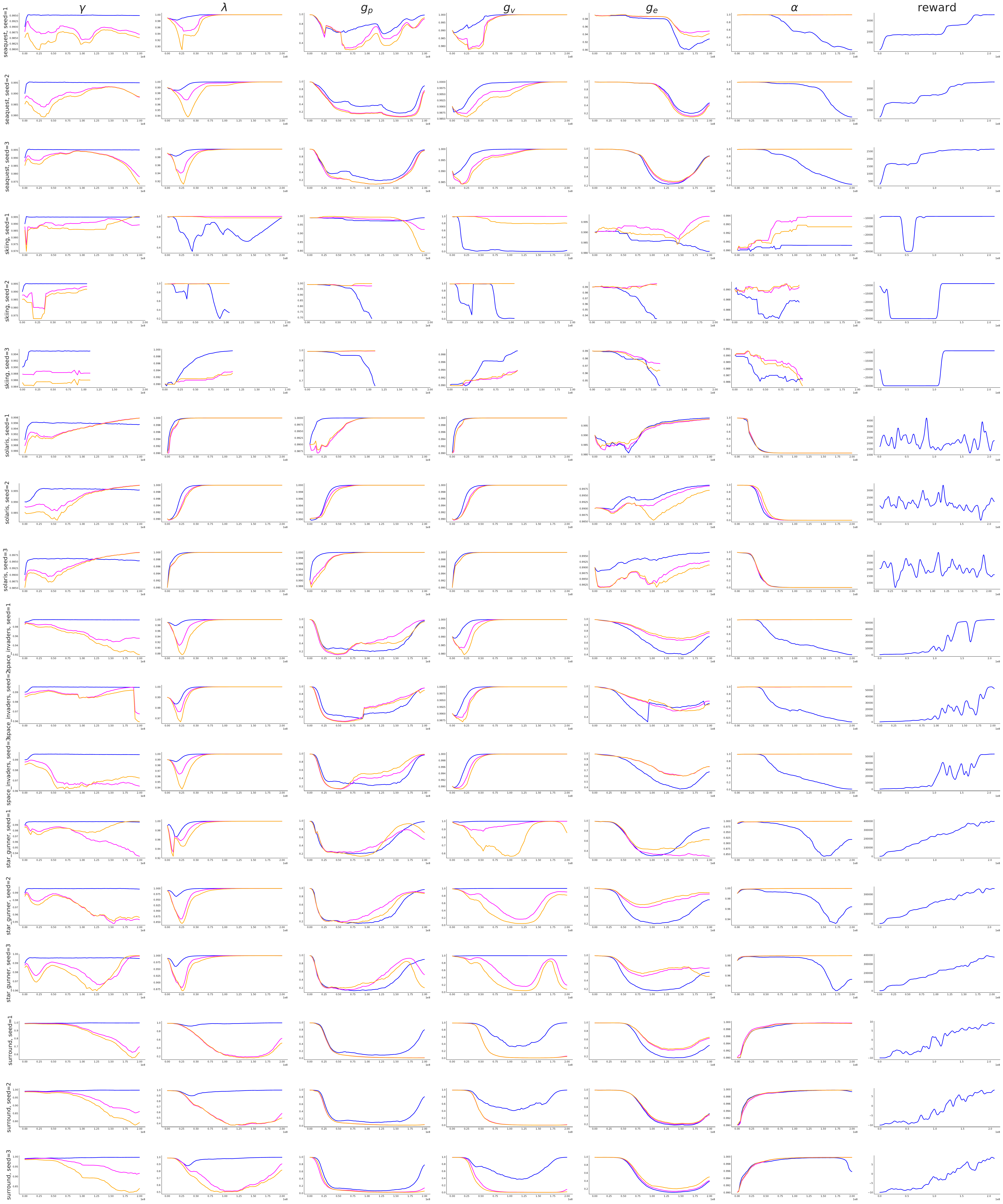}
\caption{Meta parameters and reward in each Atari game (and seed) during learning. Different colors correspond to different heads, blue is the main (policy) head.}
\end{figure}
\begin{figure}[h]
    \centering
    \includegraphics[width=\linewidth]{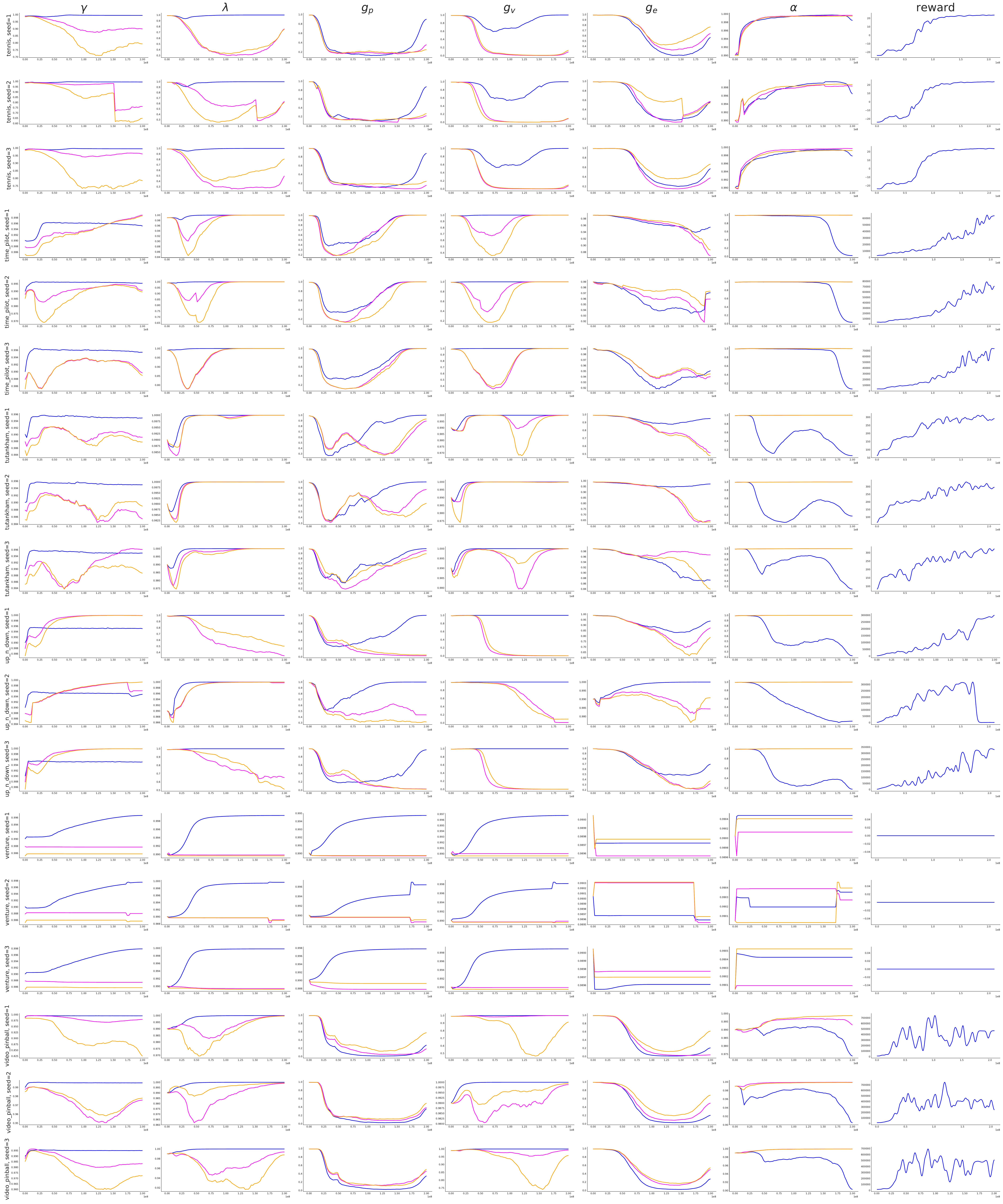}
\caption{Meta parameters and reward in each Atari game (and seed) during learning. Different colors correspond to different heads, blue is the main (policy) head.}
\end{figure}
\begin{figure}[h]
    \centering
    \includegraphics[width=\linewidth]{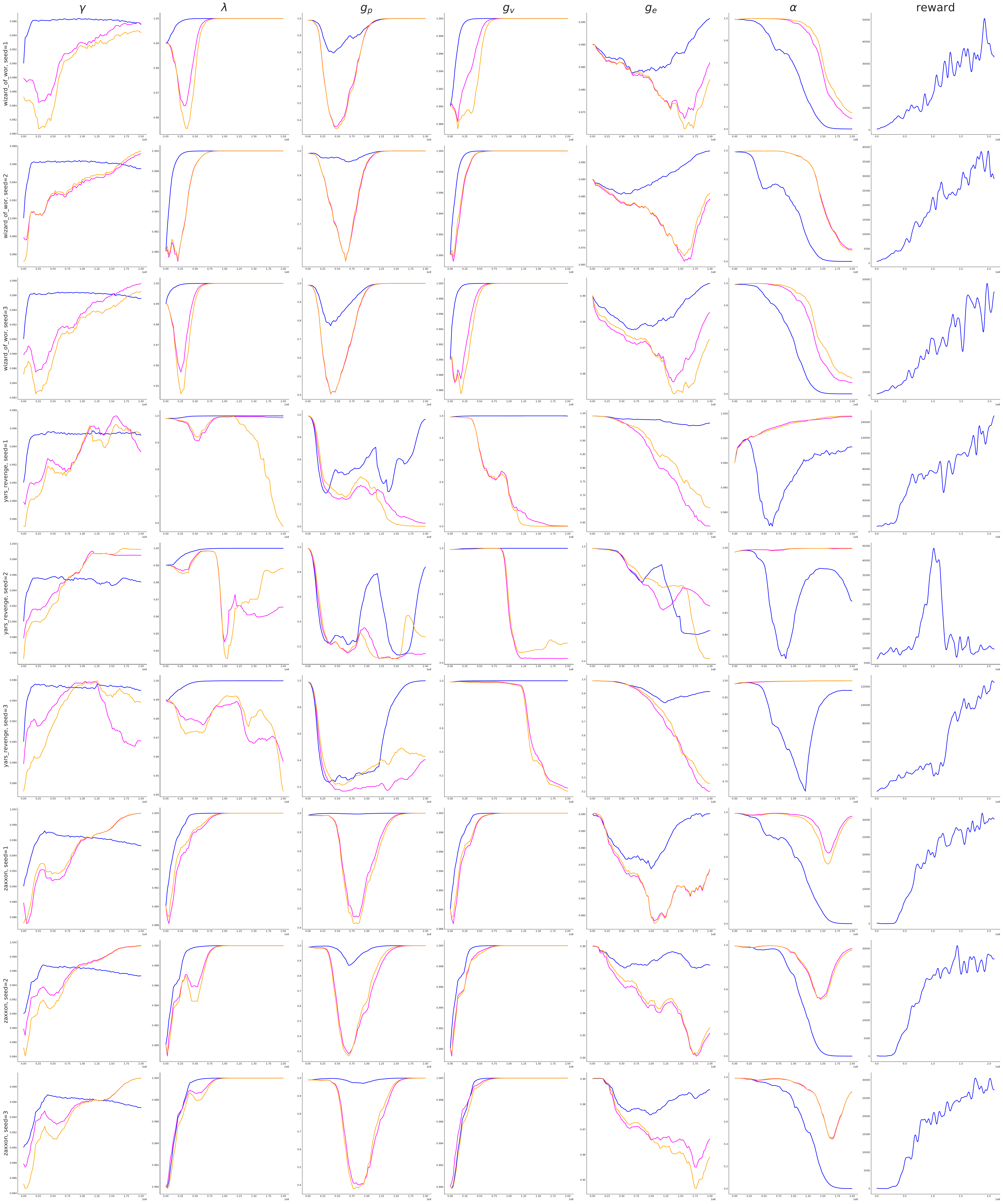}
\caption{Meta parameters and reward in each Atari game (and seed) during learning. Different colors correspond to different heads, blue is the main (policy) head.}
\end{figure}

\end{document}